\documentclass[11pt,letterpaper]{article}

\renewcommand{\vec}[1]{\mathbf{#1}}

\usepackage{subfigure}
\usepackage{algpseudocode}
\usepackage{amsmath, amsthm, amssymb, amstext, comment, graphicx}
\usepackage{fullpage}
\usepackage{mathtools,extarrows}
\usepackage{hyperref}
\usepackage{bbm}
\usepackage{array}
\usepackage{comment}
\newcolumntype{P}[1]{>{\centering\arraybackslash}p{#1}}

\usepackage[ruled,linesnumbered]{algorithm2e} 

\SetAlFnt{\small}
\SetAlCapFnt{\small}
\SetAlCapNameFnt{\small}
\SetAlCapHSkip{0pt}
\IncMargin{-\parindent}
\usepackage{xcolor}
\usepackage{dsfont}

\newcommand{\norm}[1]{\left\| #1 \right\|}
\newcommand{\abss}[1]{\left\lvert {#1} \right\rvert}
\newcommand*\samethanks[1][\value{footnote}]{\footnotemark[#1]}
\newenvironment{prevproof}[2]{\noindent {\em {\textbf{Proof of {#1}~\ref{#2}:}}}}{$\hfill\qed$\vskip \belowdisplayskip} 

\newtheorem{theorem}{Theorem}[section]

\newtheorem{corollary}{Corollary}[section]

\newtheorem{lemma}{Lemma}[section]

\newtheorem{remark}{Remark}[section]

 \allowdisplaybreaks

\title{Regression from Dependent Observations}
\author{
	Constantinos Daskalakis
	\thanks{Supported by NSF awards CCF-1617730 and IIS-1741137, a Simons Investigator Award, a Google Faculty Research Award, and an MIT-IBM Watson AI Lab research grant.}\\
	EECS \& CSAIL, MIT\\
	\tt{costis@csail.mit.edu}
	\and
	 Nishanth Dikkala
	\samethanks \\
	EECS \& CSAIL, MIT\\
	\tt{nishanthd@csail.mit.edu}
	\and
	Ioannis Panageas
	\thanks{Supported by SRG ISTD 2018 136. Part of this work was done while the authors were visiting Northwestern University for a spring program in Econometrics}\\
	ISTD \& SUTD\\
	\tt{ioannis@sutd.edu.sg}
}

\begin{document}
	\maketitle


\begin{abstract}
The standard linear and logistic regression models assume that the response variables are  independent, but share the same linear relationship to their corresponding vectors of covariates. The assumption that the response variables are independent is, however, too strong. In many applications, these responses are collected on nodes of a network, or some spatial or temporal domain, and are dependent. Examples abound in financial and meteorological applications, and dependencies naturally arise in social networks through peer effects. Regression with dependent responses has thus received a lot of attention in the Statistics and Economics literature, but there are no strong consistency results unless multiple independent samples of the vectors of dependent responses can be collected from these models. We present computationally and statistically efficient methods for linear and logistic regression models when the response variables are dependent on a network. Given one sample from a networked linear or logistic regression model and under mild assumptions, we prove strong consistency results for recovering the vector of coefficients and the strength of the dependencies, recovering the rates of standard regression under independent observations. We use projected gradient descent on the negative log-likelihood, or negative log-pseudolikelihood, and establish their strong convexity and consistency using concentration of measure for dependent random variables.
\end{abstract}
\thispagestyle{empty}
\newpage
\setcounter{page}{1}
\section{Introduction}
\label{sec:intro}

Linear and logistic regression are perhaps the two most prominent models in Statistics. In their most standard form, these models postulate that a collection of {\em response variables} $y_1,\ldots, y_n$, which are scalar and binary respectively, are linearly related to a collection of {\em covariates} $\vec{x}_1,\ldots,\vec{x}_n \in \mathbb{R}^d$ through some {\em coefficient vector} $\theta$, as follows:

\begin{itemize}
\item in vanilla {\em linear regression} it is assumed that:
\begin{itemize}
\item for all $i \in \{1,\ldots,n\}$: $y_i = \theta^{\top}\vec{x}_i + \epsilon_i$, \\
where $\epsilon_i \sim {\mathcal N}(0,1)$; and
\item $y_1,\ldots,y_n$ are independent.
\end{itemize}

\item {in vanilla {\em logistic regression} it is assumed that:}
\begin{itemize}
\item for all $i \in \{1,\ldots,n\}$ and $\sigma_i \in \{\pm 1\}$: $\Pr[y_i = \sigma_i] = {1 \over 1 + \exp\left( -2\theta^{\top} \vec{x_i} \sigma_i\right)}$; and
\item $y_1,\ldots,y_n$ are independent.
\end{itemize}
\end{itemize}

%
\noindent It is well-known that, given examples $(\vec{x}_i,y_i)_{i=1}^n$, where the $y_i$'s are sampled independently as specified above, the coefficient vector $\theta$ can be estimated to within $\ell_2$-error $O_d\left({\sqrt{1 \over n}}\right)$ in both models, under mild assumptions about the smallest singular value of the matrix whose rows are  $\vec{x}_1,\ldots,\vec{x}_n$. In both cases, this can be achieved by solving the corresponding Maximum Likelihood Estimation (MLE) problem, which is concave. In fact, in linear regression, the optimum of the likelihood has a closed form, which is the familiar least-squares estimate.

\medskip The assumption that the response variables $y_1,\ldots,y_n$ are independent is, however, too strong. In many applications, these variables are observed on nodes of a network, or some spatial or temporal domain, and are {\em dependent}. Examples abound in financial and meteorological applications, and dependencies naturally arise in social networks through {\em peer effects}, whose study has recently exploded in topics as diverse as criminal activity (see e.g.~\cite{glaeser1996crime}), welfare participation (see e.g.~\cite{bertrand2000network}), school achievement (see e.g.~\cite{sacerdote2001peer}), participation in retirement plans~\cite{duflo2003role}, and obesity (see e.g.~\cite{trogdon2008peer,christakis2013social}). A prominent dataset where peer effects have been studied are  data collected by the National Longitudinal Study of Adolescent Health, a.k.a.~AddHealth study~\cite{harris2009waves}. This was a major national study of students in grades 7-12, who were asked to name their friends---up to 10, so that friendship networks can be constructed, and answer hundreds of questions about their personal and school life, and it also recorded information such as the age, gender, race, socio-economic background, and health of the students. Estimating models that combine peer and individual effects to predict behavior in such settings has been challenging; see~e.g.~\cite{manski1993identification,bramoulle2009identification}.

\subsection{Modeling Dependence}
\medskip In this paper, we generalize the standard linear and logistic regression models to capture dependencies between the response variables, and show that if the dependencies are sufficiently weak, then both the coefficient vector $\theta$ and the strength of the dependencies among the response variables can be estimated to within error $O_d\left({\sqrt{1 \over n}}\right)$. To define our models, we drop the assumption that the response variables $y_1,\ldots,y_n$ are independent, but maintain the form of the conditional distribution that each response variable $y_i$ takes, conditioning on a realization of the other response variables $y_{-i}$. In particular, for all $i$, conditioning on a realization of all other  variables $y_{-i}$, the conditional distribution of $y_i$:
\begin{itemize}
\item (in our linear regression model) is a Gaussian of variance $1$, as in standard linear regression, except that the mean of this Gaussian may depend on both $\theta^{\top}\vec{x}_i$ and in some restricted way the realizations $y_{j}$ and the covariates $\vec{x}_j$, for $j \neq i$;
\item (in our logistic regression model) takes value $+1$ with probability computed by the logistic function, as in standard logistic regression, except that the logistic function is evaluated at a point that may depend on both $\theta^{\top}\vec{x}_i$ and in some restricted way the realizations $y_{j}$ and covariates $\vec{x}_j$, for $j \neq i$.
\end{itemize}

To capture network effects we parametrize the afore-described general models through a (known) interaction matrix $A \in \mathbb{R}^{n \times n}$ and an (unknown) strength of interactions $\beta\in \mathbb{R}$, as follows.
\begin{itemize}
\item In {\em linear regression} with $(A,\beta)$-dependent samples we assume that:
\begin{itemize}
\item $\vec{\epsilon} = \vec{y} - X \theta$, with $\vec{\epsilon} \sim {\mathcal N}(\vec{0},(\beta A+D)^{-1})$.
\item Or equivalently, for all $i$, conditioning on a realization of the response variables $y_{-i}$:
\begin{align}
y_i = \theta^{\top}\vec{x}_i  + \epsilon_i, \label{eq:costas linear2}
\end{align}
where $\epsilon_i \sim {\mathcal N}(\Sigma_{i}\Sigma_{ii}^{-1}\vec{\alpha}_i, \frac{\det ((\beta A + D)_{-i})}{\det (\beta A+D)} -\Sigma_{i}\Sigma_{ii}^{-1}\Sigma_{i}^{\top})$, where $\Sigma_{i}$ is the i-th row of $(\beta A+D)^{-1}$ by removing the coordinate (diagonal element) $i$-th, $\Sigma_{ii}$ is $(\beta A+D)^{-1}$ by removing the i-th column and $i$-th row, $(\beta A + D)_{-i}$ is $\beta A+D$ by removing $i$-th row and column and finally column vector $\alpha_j = y_j - \theta^{\top}\vec{x}_j $ (this is the Schur complement for conditional multivariate Gaussians). Observe that $\Sigma_{i}\Sigma_{ii}^{-1} = -\frac{1}{D_{ii}} \beta A_{i}$\footnote{$A_i$ denotes the $i$-row of $A$ by removing coordinate $i$, i.e., $n-1$ vector} and hence the expectation becomes $-\frac{1}{D_{ii}} \sum_{j \neq i} \beta A_{ij}(y_{j} - \theta^{\top}\vec{x}_j)$ and moreover the variance becomes $\frac{1}{D_{ii}}$. By the transformation $\epsilon'_i = \epsilon_i + \frac{1}{D_{ii}}\sum_{j \neq i} \beta A_{ij}(y_{j} - \theta^{\top}\vec{x}_j)$ we get that \begin{align}
y_i = \theta^{\top}\vec{x}_i  - \frac{1}{D_{ii}}\left[ \sum_{j \neq i} \beta A_{ij}(y_{j} - \theta^{\top}\vec{x}_j)\right] + \epsilon'_i, \label{eq:costas linear}
\end{align}
with $\epsilon'_i \sim {\mathcal N}\left(0,\frac{1}{D_{ii}}\right)$.
\item Interpretation: The conditional expectation of $y_i$ is additively perturbed from its expectation $\theta^{\top} \vec{x}_i$ by the weighted average, according to weights $\beta A_{ij}$, of how much the other responses are perturbed from their expectations in realization $y_{-j}$.


\item Remark 2: The model proposed in Eq.~\eqref{eq:costas linear} falls in the realm of auto-regressive models studied by Manski~\cite{manski1993identification} and Bramoull\'e et al.~\cite{bramoulle2009identification}, where it is shown that the model can be identified under conditions on the interaction matrix $A$. In contrast to our work, one of the conditions imposed on $A$ is that it can be partitioned into many identical blocks (i.e.~the weighted graph defined by $A$ has many identical connected components). Thus the response variables cluster into multiple groups that are independently and identically sampled, given the covariates. Instead we want to identify $\theta$ and $\beta$ even when $A$ corresponds to one strongly connected graph, and therefore there is no independence to be exploited.
\end{itemize}

\item In {\em logistic regression} with $(A,\beta)$-dependent samples it is assumed that:
\begin{itemize}
\item For all $i$ and $\sigma_i \in \{\pm 1\}$, conditioning on a realization of the response variables $y_{-i}$:

\begin{align}
\Pr[y_i = \sigma_i] &= {1 \over 1 + \exp\left( -2\left(\theta^{\top} \vec{x_i} + \beta \sum_{j \neq i}A_{ij}y_j\right)\sigma_i\right)}. \label{eq:costas logistic}
\end{align}
\item Interpretation: The probability that the conditional distribution of $y_i$ assigns to $+1$ is determined by the logistic function applied to $2\left(\theta^{\top} \vec{x_i}  + \beta \sum_{j \neq i}A_{ij}y_j\right)$ instead of $2\theta^{\top} \vec{x_i}$, i.e.~it is increased by the weighted average, according to weights $\beta A_{ij}$, of the other responses in realization $y_{-j}$.

\item Remark 3: It is easy to see that the joint distribution of random variables $(y_1,\ldots,y_n)$, satisfying the requirements of Eq.~\eqref{eq:costas logistic}, is an instance of the Ising model. See Eq.~\eqref{eq:logisticmodel}. In this Ising model each variable $i$ has external field $\theta^{\top}\vec{x}_i$, and $\beta$ controls the inverse temperature of the model.  The Ising model was originally proposed to study phase transitions in spin systems~\cite{Ising25}, and has since found myriad  applications in diverse research disciplines, including probability theory, Markov chain Monte Carlo, computer vision, theoretical computer science, social network analysis, game theory, and computational biology~\cite{LevinPW09,Chatterjee05,Felsenstein04,DaskalakisMR11,GemanG86,Ellison93,MontanariS10}.

\medskip A particularly simple instance of our model arises when all covariates $\vec{x}_i$ are single dimensional and identical.  In this case, our model only has two free parameters, and this setting has been well-studied. \cite{comets1991asymptotics}  consider the consistency of maximum likelihood estimation in this setting.  More recent work of Chatterjee~\cite{chatterjee2007estimation}, Bhattacharya and Mukherjee~\cite{bhattacharya2018inference}, and Ghosal and Mukherjee~\cite{ghosal2018joint} has identified conditions on the interaction matrix $A$ under which these parameters can be identified. Our work generalizes these works to the case of multi-dimensional covariates.

\end{itemize}
\end{itemize}

Now let us state our results for the above regression models with dependent response variables. We are given a set of observations $(\vec{x}_i,y_i)_{i=1}^n$ where the covariates $\vec{x}_i$ are deterministic, and the response variables are assumed to have been sampled according to either of the models above, for a given interaction matrix $A$ and an unknown scalar $\beta$ and coefficient vector $\theta$. Given our observations, we are interested in estimating $\beta$ and $\theta$. {\em It is important to stress that we only have {\bf one sample} of the variables $(y_1,\ldots,y_n)$.} In particular, we cannot redraw the response variables many times and derive statistical power from the independence of the samples. This is motivated by our application to network collected data, where we often have no access to independent snapshots of the responses at the nodes of the network. On a technical standpoint, estimating from a {\em single sample} distinguishes our work from  other works in the literature of auto-regressive models and graphical models, and requires us to deal with the challenges of concentration of measure of functions of dependent random variables.

Our main results are stated as Theorems~\ref{thm:logistic}, for logistic regression, and~\ref{thm:linear}, for linear regression. In both cases, the parameters $\beta, \theta$ can be estimated to within error $O_d\left(\sqrt{1 \over n}\right)$, the dependence of the rate on $n$ matching that of vanilla logistic regression and vanilla linear regression respectively \footnote{The dependence on $d$ in the rate is $O(\exp(2d))$ for both vanilla regression and our result. Hence we achieve the optimal rate with respect to all parameters in the setting of logistic regression. For linear regression with dependent samples, our rate dependence on $d$ is $O(d^{3/2})$ which is off by the rate achieved for independent samples by a factor of $d$. }. These results hold under the assumptions of Table~\ref{table:assumptions}. We note that the assumptions on $\theta$, $\beta$, and the covariates are standard, even in the case of vanilla regression. Moreover, the bounds on the norm of $A$ have been shown to be necessary for logistic regression by~\cite{bhattacharya2018inference,ghosal2018joint}. And the minimum singular value condition for matrix $AX$ is mild, and holds for various ensembles of $A$; see e.g.~Corollary~\ref{cor:sk} shown using Ky Fan inequalities~\cite{kyfan}.

\paragraph{Proof Overview:} The estimation algorithms in both Theorem~\ref{thm:logistic} and Theorem~\ref{thm:linear} are instances of Projected Gradient Descent (PGD). In the linear case (Theorem~\ref{thm:linear}, PGD is applied to the negative log-likelihood of the observations $(y_1,\ldots,y_n)$. However, the log-likelihood is not convex, so we perform a re-parametrization of the model, indeed an {\em overparametrization} of the model that renders it convex. Showing strong convexity of the re-parametrized negative log-likelihood requires some mild linear algebra. It has to be established that despite the overparametrization the optimum collapses to the right dimensionality, and can be used to recover the original parameters. A more complete overview of the approach is presented in the beginning of Section~\ref{sec:linear}.

In the logistic case (Theorem~\ref{thm:logistic}), we do not run PGD on the negative log-likelihood but the negative\\ log-{\em pseudo}likelihood. Pseudolikelihood is the product of the conditional probabilities of each response $y_i$, conditioning on all other responses $y_{-i}$. Pseudolikelihood is trivially convex, but we need to establish that is optimum is close to the true parameters and also that it is strongly convex. We show both properties via concentration results for functions of  dependent random variables. To show that the maximum of the pseudolikelihood is close to the true parameters we use exchangeable pairs, adapting~\cite{chatterjee2016nonlinear}. To show that it is strongly convex we show additional properties of $A$ which are implied by our assumptions. Combining these with a new concentration inequality, we obtain the desired bound. A more complete overview of the approach is presented in Section~\ref{sec:logistic-overview}.

\paragraph{Other Related Work:} We have already reviewed the work that is most relevant to ours from the Economics, Probability Theory, and Statistics literature. Further discussion of the Econometrics and Statistics literature on the theory and applications of regression with dependent observations is discussed in~\cite{li2016prediction}. There is another strand of literature studying generalization bounds that can be attained when learning from sequences of dependent observations; see~e.g.~\cite{modha1996minimum,meir2000nonparametric,zou2009generalization,steinwart2009fast,mohri2010stability,agarwal2013generalization}. These works assume, however, that the sequence of observations is a {\em stationary process}, which does not hold in our models, and they impose strong mixing conditions on that sequence. Finally, we note that generalized linear regression, which accommodates dependencies among the response variables, cannot be applied directly to our linear regression setting to estimate $\theta$, because the covariance matrix of our response variables depends on the parameter $\beta$, which is unknown and thus needs to be disentangled before bounding the error in the estimation of $\theta$.

In the case of logistic regression, there has been a lot of work showing that under certain \emph{high-temperature} conditions on the Ising model (which are similar to the assumptions we make in our paper), one can perform many statistical tasks such as learning, testing and sampling of Ising models efficiently \cite{klivans2017learning,daskalakis2017concentration,daskalakis2018hogwild,daskalakis2018testing,hamilton2017information,de2016ensuring}.

%
%
%
\section{Preliminaries}
\label{sec:prelim}
We use bold letter such as $\vec{x}, \vec{y}$ to denote vectors and capital letters $X,Y,A,D$ to denote matrices. All vectors are assumed to be column vectors, i.e. $\text{dim} \times 1$. We will refer to $A_{ij}$ as the $(i,j)^{th}$ entry of matrix $A$.
We will use the following matrix norms. For a $n \times n$ matrix $A$,
\begin{align*}
\norm{A}_2 = \max_{\norm{x}_2 = 1} \norm{Ax}_2, &\: \: \: \: \: \: \: \norm{A}_{\infty} = \max_{j \in [n]} \sum_{i=1}^n \abss{A_{ij}} , \\
& \norm{A}_F = \sqrt{\sum_{i=1}^n \sum_{j=1}^n A_{ij}^2}.
\end{align*}
When $A$ is a symmetric matrix we have that $\norm{A}_2 \le \norm{A}_{\infty} \le \norm{A}_F \le \sqrt{n}\norm{A}_2 \le \sqrt{n}\norm{A}_\infty$.
We use $\lambda$ to denote eigenvalues of a matrix and $\sigma$ to denote singular values. $\lambda_{\min}$ refers to the smallest eigenvalue and $\lambda_{\max}$ to the largest, and similar notation is used for the singular values as well.

We will say an estimator $\hat{\theta}_n$ is consistent with a rate $\sqrt{n}$ (or equivalently $\sqrt{n}$-consistent) with respect to the true parameter $\theta_0$ if there exists an integer $n_0$ and a constant $C > 0$ such that for every $n > n_0$, with probability at least $1 - o(1)$,
\begin{align*}
\norm{\hat{\theta}_n - \theta_0 }_2 \le  \frac{C}{\sqrt{n}}
\end{align*}

We utilize the following two well-known examples of graphical models to characterize dependencies in our logistic and linear regression models respectively.\\

\begin{enumerate}
	\item \textbf{Ising Model:} Given an unweighted undirected graph $G(V,E)$ with adjacency matrix $A$ and assignment $\vec{\sigma} : V \to \{-1,+1\}^n$, an Ising model is the following probability distribution on the $2^n$  configurations of $\vec{\sigma}$:
\begin{equation}
\Pr\{\vec{y} = \sigma\} = \frac{\exp\left(\sum_{v \in V}h_v \sigma_v + \beta\vec{\sigma}^{\top}A \vec{\sigma}\right)}{Z_G(\beta,\theta)} \label{eq:ising}
\end{equation}
where $$Z(G) = \sum_{\vec{\tilde{\sigma}}} \exp\left(\sum_{v \in V}h_v \tilde{\sigma}_v + \beta\vec{\tilde{\sigma}}^{\top}A \vec{\tilde{\sigma}}\right)$$ is the partition function of the system (or renormalization factor). Moreover the term $\sum_v h_v \sigma_v$ is called the external field. It can be observed that, without loss of generality, we can restrict the matrix $A$ to have zeros on its diagonal. \\

\item \textbf{Gaussian Graphical Model:}
Let $G=(V,E)$ be an undirected graph with $V= [n]$. A random vector $X \in \mathbb{R}^n$ is said to be distributed according to (undirected) Gaussian Graphical model with graph $G$ if $X$ has a multivariate Gaussian distribution $\mathcal{N}(\mu, \Sigma)$ with
\begin{align}
\left(  \Sigma^{-1} \right)_{ij} = 0 \; \; \forall \: (i,j) \notin E,
\end{align}
where the density function $f_{\mu,\Sigma}(.)$ of $\mathcal{N}(\mu,\Sigma)$ is
\begin{align*}
f_{\mu,\Sigma}(x) = \frac{\exp\left(  - \frac{1}{2}(x - \mu)^T \Sigma^{-1} (x - \mu)  \right)}{(2\pi)^{n/2} \det(\Sigma)^{1/2}}
\end{align*}
under the condition that $\Sigma$ is positive semi-definite ($\Sigma^{-1}$ is also known as the precision matrix). \\
\end{enumerate}

\subsection{Some Useful Lemmas from Literature}
Weyl's inequalities are useful to understand how the spectra of symmetric matrices change under addition. We state them here for reference.
\begin{lemma}[Weyl's Inequalities]
	\label{lem:weyl}
	Let $A$, $B$ and $C$ be three $n \times n$ symmetric matrices with real entries such that $A = B + C$. Let $\lambda_1^A \ge \lambda_2^A \ge \ldots \ge \lambda_n^A$, $\lambda_1^B \ge \lambda_2^B \ge \ldots \ge \lambda_n^B$,  $\lambda_1^C \ge \lambda_2^C \ge \ldots \ge \lambda_n^C$ be their eigenvalues respectively.
	Then we have for all $i \in [n]$, $\lambda_i^B + \lambda_n^C \le \lambda_i^A \le \lambda_i^B + \lambda_1^C$.
\end{lemma}

We will use the following concentration inequality which is standard in literature.
\begin{theorem}[\cite{vershynin2010introduction}, Remark 5.40]\label{thm:feature}
	Assume that $X$ is an $n \times d$ matrix whose rows $X_i$ are independent sub-gaussian random vectors in $R^d$ with
	second moment matrix $\Sigma$. Then for every $t \geq 0$, the following inequality holds with probability at least $1 - 2 \exp(-ct^2)$,
	\[\norm{\frac{1}{n} X^{\top}X - \Sigma}_2 \leq \max(\delta,\delta^2)\] with $\delta = C \sqrt{\frac{d}{n}} + \frac{t}{\sqrt{n}}$.
\end{theorem}
\begin{remark}\label{rem:feature} By choosing $t$ to be $\Theta(\sqrt{\ln n})$, it follows that with probability $1 - \frac{1}{\textrm{poly}(n)}$ we get that
	\[\norm{\frac{1}{n} \sum_{i=1}^n \vec{x}_i \vec{x}_i^{\top} - Q}_2 \textrm{ is } O\left(\sqrt{\frac{\ln n}{n}}\right),\] from which follows that $\lambda_{\min}(\frac{1}{n}\sum_{i=1}^n \vec{x}_i\vec{x}_i^{\top})$ is at least $\lambda_{\min}(Q) - O\left(\sqrt{\frac{\ln n}{n}}\right)$ with probability $1 - \frac{1}{\textrm{poly}(n)}$ (by Weyl's inequality).
\end{remark}

\begin{lemma}[Useful Inequalities on Singular Values]\label{lem:usefulineq}
	The following inequalities hold:
	\begin{enumerate}
		\item Let $W$ be a $n \times n $ matrix. It holds that $|\lambda_{\min}(W+W^{\top})| \leq 2 \sigma_{\min}(W)$ (see \cite{kyfan}).
		\item Let $W, Z$ be matrices. It holds that $\sigma_{\min}(WZ) \leq \sigma_{\min}(W)\norm{Z}_2$ (folklore).
		\item Let $W,Z$ be matrices, then $\norm{WZ}_F^2 \leq \norm{W}_2^2 \norm{Z}_F^2$ (folklore).
	\end{enumerate}
\end{lemma}

\begin{lemma}[Expectation and Variance of a Quadratic form of a Gaussian Distribution]\label{lem:quadratic}
	Let $\vec{z} \sim \mathcal{N}(\mu,\Sigma)$ and we have the quadratic form $f(\vec{z}):=\vec{z}^{\top}A\vec{z} + \vec{b}^{\top}\vec{z} + c$. It holds that
	\[\mathbb{E}_z[f(\vec{z})] = \mathrm{tr}(A\Sigma)+f(\mu), \mathbb{V}_z[f(\vec{z})] = 2\mathrm{tr}(A\Sigma A\Sigma)+4\mu^{\top}A\Sigma A\mu +4\vec{b}^{\top}\Sigma A\mu + \vec{b}^{\top}\Sigma \vec{b} .\]
\end{lemma}

Table~\ref{table:assumptions} lists the assumptions under which our main theorems for logistic and linear regression hold. 

\begin{table*}[ht!]
     \centering
     \caption{List of conditions under which our main consistency results (Theorems \ref{thm:logistic} and \ref{thm:linear}) hold.}
\bgroup
\def\arraystretch{1.5}
\begin{tabular}{ |P{5.0cm}|P{4.5cm}|P{5.8cm}|}
 \hline
 \textbf{Parameter} & \textbf{Logistic} & \textbf{Linear}\\
 \hline
$\theta$ & $(-\Theta, \Theta)^d$ & $(-\Theta, \Theta)^d$\\
\hline
$\vec{x}_{i}$ feature vectors with covariance matrix $Q = n^{-1} X^{\top}X$ \footnote{If $\vec{x}_i$ are drawn from a subgaussian with second moment matrix $Q'$ of size $d\times d$, the assumptions on the eigenvalues are for $Q'$ and carry over to $\frac{1}{n}X^{\top}X$ with probability $1-o(1)$.}  & Support in $[-M,M]^d$ and $\lambda_{\max}(Q), \lambda_{\min}(Q)$ positive constants & No restriction in the support and $\lambda_{\max}(Q), \lambda_{\min}(Q)$ positive constants\\
\hline
$D$ & Not Applicable & diagonal matrix with positive constant entries\\
\hline
$A$ & symmetric, zero diagonal, $\norm{A}_{\infty} \leq 1$ and $\norm{A}_F^2 \ge cn$ & symmetric, zero diagonal, $\norm{A}_{2} \leq 1$ and and $\norm{A}_F^2 \ge cn$ \\
  \hline
$\beta$ & $(-B, B)$ & $\lambda_{\min}((\beta A+D)^{-1})>\rho_{\min}$, $\lambda_{\max}((\beta A+D)^{-1})<\rho_{\max}$ and $\rho_{\min}, \rho_{\max}$ positive constants for all $\beta \in (-B,B)$\\
  \hline
    \vspace{0.2pt}
   $n^{-1} X^{\top}A^{\top}(I - DX(X^{\top}D^2X)^{-1}X^{\top}D) AX$ & \vspace{0.2pt}No assumption & \vspace{0.2pt}Minimum eigenvalue a positive constant $\rho_{DAX}$ \\
   \hline
\end{tabular}
\egroup
 \label{table:assumptions}
\end{table*}

\section{Logistic Regression with Dependent Data}\label{sec:logistic}
In this section we look at the problem of logistic regression with dependent data.

\subsection{Our model}
We are interested in a generalization of the Ising model on graph $G = (V,E)$ with $|V| = n$, where each vertex $i \in G$ has a feature vector $\vec{x}_i \in \mathbb{R}^d$. Moreover there is an unknown parameter $\theta \in \mathbb{R}^d$ and the corresponding probability distribution induces to the following:
\begin{equation}
\Pr\{\vec{y} = \sigma\} = \frac{\exp\left(\sum_{i=1}^n(\theta^{\top}\vec{x}_i) \sigma_i + \beta\vec{\sigma}^{\top}A \vec{\sigma}\right)}{Z(G)}, \label{eq:logisticmodel}
\end{equation}
where $A$ is a symmetric matrix with zeros on the diagonal.
Given one sample $\vec{y}$ and the knowledge of the matrix $A$, we would like to infer $\beta, \theta$.

We now study some conditions under which we can attain consistent estimates of the parameters of the model. Combined with some standard assumptions on the data-generating process of the feature vectors all our assumptions are listed in Table~\ref{table:assumptions}.
\begin{theorem}[Logistic Regression with Dependent Samples]
	\label{thm:logistic}
	Consider the model of (\ref{eq:logisticmodel}). The Maximum Pseudo-Likelihood Estimate (MPLE) $(\hat{\theta}_{MPL},\hat{\beta}_{MPL})$ is consistent with a rate of $\sqrt{n}$ as long as $(\theta_0,\beta_0)$ and the features $X$ satisfy the conditions of Column 2 in Table \ref{table:assumptions}. Formally, for each constant $\delta>0$ and $n$ sufficiently large
	\begin{align*}
	\norm{(\hat{\theta}_{MPL},\hat{\beta}_{MPL}) - (\theta_0,\beta_0)}_2 \le O_d\left({\sqrt{1 \over n}}\right)
	\end{align*}
	with probability $1-\delta$.
	Moreover, we can compute a vector $(\tilde{\theta},\tilde{\beta})$ with $\norm{(\hat{\theta}_{MPL},\hat{\beta}_{MPL}) - (\tilde{\theta},\tilde{\beta})}_2 \le O_d\left({\sqrt{1 \over n}} \right)$ in $O(\ln n)$ iterations of projected gradient descent (Algorithm in Section \ref{sec:gd-analysis}) where each iteration takes at most $O(dn)$ time, with probability $1-\delta$.
\end{theorem}

\begin{remark}[Necessity of an Upper Bound on $\norm{A}_{\infty}$ and boundedness of $\beta_0$]
	If $\norm{A}_{\infty}$ scales with $n$ then no consistent estimator might exist. This is because the peer effects through $\beta_0 A$ will dominate the outcome of the samples and will nullify the signal coming from $\theta_0^{\top}X$. Similarly one requires $\beta_0$ to be bounded as well to preserve some signal to enable recovery of $\theta_0$.
\end{remark}

\begin{remark}[Necessity of the Lower Bound on $\norm{A}_F$]
	It was shown in \cite{bhattacharya2018inference} (Corollary 2.4 (b)) and \cite{ghosal2018joint} (Theorem 1.13) that when the condition $\norm{A}^2_F > cn$ is violated, we have specific examples where it is impossible to get consistent estimators for $(\theta_0,\beta_0)$.
	The first instance is the Curie-Weiss model $CW(n,\beta,h)$ ($A_{ij} = \frac{1}{n}$ for all $i \ne j$). Note that $\norm{A}^2_F = O(1)$ in this case.
	The second instance is dense random graphs, i.e. $G(n,p)$ where $p$ is a constant independent of $n$ and $A$ is chosen to be the adjacency matrix scaled down by the average degree of the graph, i.e. $A_{ij} = \frac{1}{(n-1)p}\mathds{1}_{(i,j) \in E}$.
\end{remark}

\begin{remark}
If the parameter $\beta_0$ is known,  the condition that $\norm{A}_F^2 \ge cn$ is not necessary for consistency of the MPL estimate $\hat{\theta}_{MPL}$.
For instance, consider the independent case where $\beta_0 = 0$. Then, to recover $\vec{\theta}$, we do not need $\norm{A}_F^2 \ge cn$.
\end{remark}

\begin{remark}
	\label{rem:sqrtd-logistic}
Our approach achieves a $\sqrt{n/d}$ rate of consistency if $\norm{\vec{x}}_2 \times \norm{\vec{\theta}}_2 = O(1)$.
\end{remark}


\noindent \textbf{Example Instantiations of Theorem \ref{thm:logistic}}
Two example settings where the conditions required for Theorem \ref{thm:logistic} to hold are satisfied are
\begin{itemize}
	\item $A$ is the adjacency matrix of graphs with bounded degree $d$ scaled down so that $\norm{A}_2 \le 1$.
	\item $A$ is the adjacency matrix of a random $d$-regular graph.
\end{itemize}

\subsection{Technical Overview}
\label{sec:logistic-overview}
Estimation in Ising models is a well-studied problem which offers a lot of interesting technical challenges. A first approach one considers is maximum likelihood estimation. However the intractability of computing the partition function poses a serious obstacle for the MLE. Even if one could approximate the partition function, proving consistency of the MLE is a hard task. To circumvent these issues we take a maximum pseudo-likelihood approach. This was proposed by Julian Besag \cite{besag1975statistical} and analyzed for inference problems on Ising models by Chatterjee \cite{chatterjee2007estimation} and others (\cite{bhattacharya2018inference},\cite{ghosal2018joint}).
Given a sample of response variables $\vec{y}$ let $f_i(\theta,\beta,\vec{y})$ denote the condition likelihood of observing $y_i$ conditioned on everyone else. The pseudo-likelihood estimator of $\vec{y}$ is
\begin{align}
\left(\hat{\theta}_{MPL}, \hat{\beta}_{MPL}\right) = \text{argmax}_{\theta,\beta} \prod_{i=1}^n f_i(\theta,\beta,\vec{y}). \label{eq:pseudolikelihood}
\end{align}
This does away with the problematic partition function and retains concavity in the parameters $\theta,\beta$. To show that the MPLE is consistent we need to show that its global optimum $(\hat{\theta}_{MPL}, \hat{\beta}_{MPL})$ is close in $\ell_2$ distance to $(\theta_0, \beta_0)$. We achieve this by showing two things hold simultaneously.
\begin{itemize}
	\item The log pseudo-likelihood is strongly concave everywhere. This will tell us that the gradient of the log pseudo-likelihood quickly increases as we move away from $(\hat{\theta}_{MPL}, \hat{\beta}_{MPL})$ where it is 0.
	\item The norm of the gradient of the log pseudo-likelihood is small at when evaluated at $(\theta_0,\beta_0)$ hence implying proximity to the MPL estimates due to strong concavity.
\end{itemize}
We show that both these conditions are satisfied with high probability over the draw of our samples. Showing that the norm of the gradient is bounded involves obtaining variance bounds on two functions of the Ising model (Lemmas \ref{lem:conc1} and \ref{lem:conc2}), and showing strong concavity amounts to showing a linear in $n$ lower bound on a particular quadratic function (see initial steps of proof in Lemma \ref{lem:hessian-lb}). Both these properties are challenging to prove because of the dependences between samples.
To tackle the lack of independence, the proofs require a rich set of technical frameworks. In particular, to show the variance bounds we use the technique of exchangeable pairs developed by Chatterjee \cite{chatterjee2005concentration}. The boundedness of $\norm{A}_{\infty}$ is necessary to have these concentration results.
To show strong concavity of the log pseudolikelihood we first prove some properties of the matrix $A$ together with an additional variance bound again shown via exchangeable pairs. The lower bound on $\norm{A}_F$ is necessary to achieve strong concavity.
Finally, we show in Section \ref{sec:gd-analysis} that computing the MPLE can be achieved efficiently using projected gradient descent where after each step we project back into the space restriced by the conditions of Table \ref{table:assumptions}. We describe each of these steps formally now.

\subsection{Analyzing the Maximum Pseudolikelihood Estimator (MPLE)}
We will treat terms not involving $n$ as constants for the purposes of our analysis. We start by analyzing the maximum pseudo-likelihood estimator. Given the feature vector of the $i^{th}$ sample $\vec{x}_i$, we denote by $x_{ik}$ the $k^{th}$ element of $\vec{x}_i$. Let $m_i(\vec{y}) := \sum_{j=1}^n A_{ij}y_j$ and let $\mathbb{B} = [-\Theta,\Theta]^d \times \left[-B , B\right]$ (the true parameters lie in the interior of $\mathbb{B}$).
The pseudolikelihood for a specific sample $\vec{y}$ is given by:
\begin{equation}
PL(\theta,\beta) := \prod_{i=1}^n \frac{\exp\left(\theta^{\top}\vec{x}_i y_i+ \beta m_i(\vec{y})y_i \right)}{\exp\left(\theta^{\top}\vec{x}_i + \beta m_i(\vec{y}) \right) + \exp\left(-\theta^{\top}\vec{x}_i - \beta m_i(\vec{y}) \right)}.
\end{equation}
The normalized log pseudolikelihood for a specific sample $\vec{y}$ is given by:
\begin{equation}
\begin{array}{ll}
LPL(\theta,\beta) := \frac{1}{n}\log PL(\theta,\beta) = -\ln 2 + \frac{1}{n}\sum_{i=1}^n \left[ y_i \beta m_{i}(\vec{y})+y_i(\theta^{\top}\vec{x}_i) - \ln \cosh (\beta m_i(\vec{y})+\theta^{\top}\vec{x}_i)\right].
\end{array}
\end{equation}

The first order conditions give:
\begin{equation}\label{eq:foc}
\begin{array}{ll}
\frac{\partial LPL(\hat{\theta}_{MPL},\hat{\beta}_{MPL})}{\partial \beta} = \frac{1}{n}\sum_{i=1}^n \left[ y_i m_{i}(\vec{y}) -  m_i(\vec{y})\tanh (\hat{\beta}_{MPL} m_i(\vec{y})+\hat{\theta}_{MPL}^{\top}\vec{x}_i)\right]=0 ,\\
\frac{\partial LPL(\hat{\theta}_{MPL},\beta)}{\partial \theta_k} =\frac{1}{n}\sum_{i=1}^n \left[ y_ix_{i,k} - x_{i,k}\tanh (\hat{\beta}_{MPL} m_i(\vec{y})+\hat{\theta}_{MPL}^{\top}\vec{x}_i)\right]=0.
\end{array}
\end{equation}

The Hessian $H_{(\theta, \beta)}$ is given by:
\begin{equation}
\begin{array}{ll}
\frac{\partial^2 LPL(\theta,\beta)}{\partial \beta^2} = -\frac{1}{n}\sum_{i=1}^n \frac{m_i^2(\vec{y})}{\cosh^2 (\beta m_i(\vec{y})+\theta^{\top}\vec{x}_i)},\\
\frac{\partial^2 LPL(\theta,\beta)}{\partial \beta \partial \theta_k} = -\frac{1}{n}\sum_{i=1}^n   \frac{x_{i,k}m_i(\vec{y})}{\cosh^2 (\beta m_i(\vec{y})+\theta^{\top}\vec{x}_i)},\\
\frac{\partial^2 LPL(\theta,\beta)}{\partial \theta_l \partial \theta_k} = -\frac{1}{n}\sum_{i=1}^n   \frac{x_{i,l}x_{i,k}}{\cosh^2 (\beta m_i(\vec{y})+\theta^{\top}\vec{x}_i)}.
\end{array}
\end{equation}
Writing the Hessian in a compact way we get  $$H_{(\theta,\beta)} = - \frac{1}{n}\sum_{i=1}^n \frac{1}{\cosh^2 (\beta m_i(\vec{y})+\theta^{\top}\vec{x}_i)}X_iX_i^{\top}$$ where $X_i = (\vec{x}_i, m_i(\vec{y}))^{\top}$. Thus $-H$ is a positive semidefinite matrix and $LPL$ is concave. Moreover if $(\theta, \beta) \in \mathbb{B}$ it follows that
\begin{equation}
\begin{array}{ll}
\frac{1}{\cosh^2 (B+d\cdot M\cdot \Theta)} \cdot \left( \frac{1}{n}\sum_{i=1}^n X_iX_i^{\top} \right)\preceq -H_{(\theta, \beta)} \preceq \left( \frac{1}{n}\sum_{i=1}^n X_iX_i^{\top} \right).
\end{array}
\end{equation}
\begin{remark}\label{rem:smooth}
Observe that $\norm{X_i}^2_2 = \norm{x_i}_2^2 + m_i^2(\vec{y}) \leq d\Theta^2 +1$ (assuming that $\norm{A}_{\infty} \leq 1$ trivially holds $|m_i(x)| \leq 1$). It is easy to see that $\lambda_{\max}(-H_{(\theta, \beta)}) \leq d\Theta^2 +1$ for all $(\theta,\beta) \in \mathbb{R}^{d+1}$, hence $-LPL$ is a $d\Theta^2+1$-smooth function, i.e. $- \nabla LPL$ is $d \Theta^2 +1$-Lipschitz.
\end{remark}
\subsection{Consistency of the MPLE}
\label{sec:logistic-consistency}

Our argument for showing consistency of the MPLE uses Lemma \ref{lem:logistic-consistency}.
\begin{lemma}
	\label{lem:logistic-consistency}
Let $(\theta_0,\beta_0)$ be the true parameter. We define $(\theta_t,\beta_t) = (1-t)(\theta_0,\beta_0)+ t(\hat{\theta}_{MPL},\hat{\beta}_{MPL})$ and let $\mathcal{D} \in [0,1]$ be the largest value such that $(\theta_{\mathcal{D}}, \beta_{\mathcal{D}}) \in \mathbb{B}$ (if it does not intersect the boundary of $\mathbb{B}$, then $\mathcal{D}=1$).
Then,
\begin{align*}
  &\norm{\nabla LPL(\theta_0,\beta_0)}_2 \ge \mathcal{D} \min_{(\theta,\beta) \in \mathbb{B}}\lambda_{\min}\left(-H_{(\theta, \beta)}\right) \norm{(\theta_0 - \hat{\theta}_{MPL},\beta_0 - \hat{\beta}_{MPL})}_2 \\
  &= \min_{(\theta,\beta) \in \mathbb{B}}\lambda_{\min}\left(-H_{(\theta, \beta)}\right) \norm{(\theta_0 - \theta_{\mathcal{D}},\beta_0 - \beta_{\mathcal{D}})}_2
\end{align*}	
\end{lemma}
\begin{proof}
We drop the subscript $MPL$ from the estimates for brevity. We set
\[g(t) := (\theta_0 - \hat{\theta},\beta_0 - \hat{\beta})^{\top} \nabla LPL(\theta_t,\beta_t),\]
\[g'(t) = -(\theta_0 - \hat{\theta},\beta_0 - \hat{\beta})^{\top} H_{(\theta_t,\beta_t)}(\theta_0 - \hat{\theta},\beta_0 - \hat{\beta}).\]
Observe that $\mathcal{D} = \frac{\norm{(\theta_{\mathcal{D}} - \theta_0,\beta_{\mathcal{D}}-\beta_0)}_2}{\norm{( \hat{\theta} - \theta_0,\hat{\beta}-\beta_0)}_2}$.
Since $H$ is negative semidefinite we have that $g'(t) \geq 0$ (*). It holds that
\begin{align*}
\norm{(\theta_0 - \hat{\theta},\beta_0 - \hat{\beta})}_2 \cdot \norm{\nabla LPL(\theta_0,\beta_0)}_2 &\geq |(\theta_0 - \hat{\theta},\beta_0 - \hat{\beta})^{\top} \nabla LPL(\theta_0,\beta_0)|  \\&=
|g(1)-g(0)| = \left|\int_{0}^1 g'(t)dt\right| \\&\geq \left|\int_{0}^{\mathcal{D}} g'(t)dt\right| \textrm{ by (*)}\\&\geq \mathcal{D}\min_{(\theta,\beta) \in \mathbb{B}}\lambda_{\min}\left(-H_{(\theta, \beta)}\right) \norm{(\theta_0 - \hat{\theta},\beta_0 - \hat{\beta})}^2_2
\\& = \min_{(\theta,\beta) \in \mathbb{B}}\lambda_{\min}\left(-H_{(\theta, \beta)}\right) \norm{(\theta_{\mathcal{D}} - \theta_0,\beta_{\mathcal{D}}-\beta_0)}_2 \times\\&\times \norm{(\theta_0 - \hat{\theta},\beta_0 - \hat{\beta})}_2
\end{align*}
\end{proof}
We apply Lemma \ref{lem:logistic-consistency}  by showing
\begin{enumerate}
	\item a concentration result for $\norm{\nabla LPL(\theta_0,\beta_0)}^2_2$ around $1/n$ (Section \ref{sec:logistic-var-bounds})and
	\item  a (positive constant) lower bound for $\min_{(\theta,\beta) \in \mathbb{B}}\lambda_{\min}\left(-H_{(\theta, \beta)}\right)$ (Section \ref{sec:logistic-strong-convexity}).
\end{enumerate}
We combine the above with the observation that $\mathcal{D} \to 1$ as $n \to \infty$ (i.e., $\mathcal{D} \geq \frac{1}{2}$ for $n$ sufficiently large). This is true because $\norm{(\theta_{\mathcal{D}} - \theta_0,\beta_{\mathcal{D}}-\beta_0)}_2 \to 0$ as $n \to \infty$ (is of order $\frac{1}{\sqrt{n}}$ by showing the promised concentration result and the lower bound). Also note that any point on the boundary of $\mathbb{B}$ has a fixed distance to $(\theta_0,\beta_0)$ since it lies in the interior. Hence $\norm{(\theta_{\mathcal{D}} - \theta_0,\beta_{\mathcal{D}}-\beta_0)}_2 \to 0$ implies that $\mathcal{D} \to 1$.  This gives the desired rate of consistency which we show in Section~\ref{sec:logistic-completing-proof}.

\subsection{Variance Bounds using Exchangeable Pairs}
\label{sec:logistic-var-bounds}
In this Section we state the lemmata which are required to show that the norm of the gradient of the log pseudo-likelihood is bounded at the true parameters.
\begin{lemma}[Variance Bound 1]
	\label{lem:conc1}
It holds that
\begin{align*}
&\mathbb{E}_{\theta_0,\beta_0} \left[ \left(\sum_{i=1}^n y_i m_i(\vec{y})- m_i(\vec{y})\tanh (\beta_0 m_i(\vec{y})+\theta_0^{\top}\vec{x}_i)\right)^2\right] \leq (12+4B)n.
\end{align*}
\end{lemma}
\begin{proof}
	We use the powerful technique of exchangeable pairs as introduced by Chatterjee (\cite{chatterjee2005concentration})  and employed by Chatterjee and Dembo (see \cite{chatterjee2016nonlinear}). First it holds that $\frac{\partial m_i(\vec{y})}{\partial y_j} =  A_{ij}$.
	Also note that since $\norm{A}_{\infty} \leq 1$ it trivially follows that $|m_i(\vec{y})| \leq 1$ for all $i$ and $\vec{y} \in \{-1,+1\}^n$.
	Set
	\begin{equation}
	Q(\vec{y}) := \sum_{i} (y_i -\tanh (\beta_0 m_i(\vec{y})+\theta_0^{\top}\vec{x}_i))m_i(\vec{y}),
	\end{equation}
	hence we get
	\begin{align}
	\frac{\partial Q(\vec{y})}{\partial y_j} = \sum_{i} \left(\vec{1}_{i=j} - \frac{\beta_0 A_{ij}}{\cosh^2(\beta_0 m_i(\vec{y})+\theta_0^{\top}\vec{x}_i)}\right)m_i(\vec{y}) + \left(y_i -\tanh (\beta_0 m_i(\vec{y})+\theta_0^{\top}\vec{x}_i)\right)\frac{\partial m_i(\vec{y})}{\partial y_j}.
	\end{align}
	We will bound the absolute value of each summand. First from above we can bound the second term as follows
	\begin{equation}\label{eq:sum2}
	\left|(y_i -\tanh (\beta_0 m_i(\vec{y})+\theta_0^{\top}\vec{x}_i))\frac{\partial m_i(\vec{y})}{\partial y_j}\right| \leq 2 \abss{A_{ij}}.
	\end{equation}
	Using the fact that $\frac{1}{\cosh^2(x)} \leq 1$ it also follows that
	\begin{equation}\label{eq:sum1}
	\left|\sum_{i} \left(\vec{1}_{i=j} - \frac{\beta_0 A_{ij}}{\cosh^2(\beta_0 m_i(\vec{y})+\theta_0^{\top}\vec{x}_i)}\right)m_i(\vec{y})\right| \leq |m_j(\vec{y})| + \sum_{i\neq j} \beta_0 \abss{A_{ij} m_i(\vec{y})}
	\end{equation}
	Using (\ref{eq:sum2}) and  (\ref{eq:sum1}) it follows that $\left|\frac{\partial Q(\vec{y})}{\partial y_j}\right| \leq \sum_{i\neq j} \abss{A_{ij}}(2+\beta_0 |m_i(\vec{y})|) + |m_j(\vec{y})|$.
	Finally let $\vec{y^j} = (\vec{y}_{-j}, -1)$ and note that
	\begin{equation}\label{eq:boundflip1}
	|Q(\vec{y}) - Q(\vec{y^j})| \leq 2 \cdot \left(\sum_{i\neq j} \abss{A_{ij}}(2+\beta_0 |m_i(\vec{y}_{-j},w)|) + |m_j(\vec{y}_{-j},w)|\right) \leq 2 \cdot (\sum_{i\neq j} \abss{A_{ij}}(2+B)+1)\leq 6+2B,
	\end{equation}
	where $w$ is the argmax of $\abss{m_i(\vec{y}_{-j},w)}$ along the line with endpoints $\vec{y}, \vec{y^j}$ (Taylor). In the last inequality we used that $\norm{A}_{\infty} \le 1$.
	We have all the ingredients to complete the proof. We first observe that
	\begin{equation}
	\sum_i \mathbb{E}_{\theta_0,\beta_0}[(y_i - \tanh (\beta_0 m_i(\vec{y})+\theta_0^{\top}\vec{x}_i))Q(\vec{y^i})m_i(\vec{y})]=0,
	\end{equation}
	since
	\begin{equation}
	\begin{array}{cc}
	\mathbb{E}_{\theta_0,\beta_0}[(y_i - \tanh (\beta_0 m_i(\vec{y})+\theta_0^{\top}\vec{x}_i))Q(\vec{y^i})m_i(\vec{y})]= \\= \mathbb{E}_{\theta_0,\beta_0}[\mathbb{E}[(y_i - \tanh (\beta_0 m_i(\vec{y})+\theta_0^{\top}\vec{x}_i))Q(\vec{y^i})m_i(\vec{y})| \vec{y}_{-i}]] =0.
	\end{array}
	\end{equation}
	Therefore it follows
	\begin{align*}
	\mathbb{E}_{\theta_0,\beta_0}[Q^2(\vec{y})] &= \mathbb{E}_{\theta_0,\beta_0}\left[Q(\vec{y}) \cdot \left(\sum_{i} (y_i -\tanh (\beta_0 m_i(\vec{y})+\theta_0^{\top}\vec{x}_i))m_i(\vec{y})\right)\right] \\& = \mathbb{E}_{\theta_0,\beta_0}\left[\sum_{i} \left(Q(\vec{y}) (y_i -\tanh (\beta_0 m_i(\vec{y})+\theta_0^{\top}\vec{x}_i))m_i(\vec{y})\right)\right] \\& = \sum_{i} \mathbb{E}_{\theta_0,\beta_0}\left[(Q(\vec{y}) - Q(\vec{y^i})) \cdot  (y_i -\tanh (\beta_0 m_i(\vec{y})+\theta_0^{\top}\vec{x}_i))m_i(\vec{y})\right] \\& \leq \sum_i 2 \cdot (6+2B) = (12+4B)n.
	\end{align*}
\end{proof}

\begin{lemma}[Variance Bound 2]
	\label{lem:conc2}
\begin{align*}
&\mathbb{E}_{\theta_0,\beta_0} \left[\sum_{k=1}^d \left (\sum_{i=1}^n x_{i,k}y_i - x_{i,k}\tanh (\beta_0 m_i(\vec{y})+\theta_0^{\top}\vec{x}_i)\right)^2 \right] \leq (4+4B)M^2 \cdot d n.
\end{align*}
\end{lemma}
\begin{proof}
	We use the powerful technique of exchangeable pairs as employed by Chatterjee and Dembo (see \cite{chatterjee2016nonlinear}).
	Note that since $\norm{A}_{\infty} \leq 1$ it trivially follows that $|m_i(\vec{y})| \leq 1$ for all $i$ and $\vec{y} \in \{-1,+1\}^n$.
	We fix a coordinate $k$ and set
	\begin{equation}
	Q(\vec{y}) := \sum_{i} (y_i -\tanh (\beta_0 m_i(\vec{y})+\theta_0^{\top}\vec{x}_i))x_{i,k},
	\end{equation}
	
	hence we get $\frac{\partial Q(\vec{y})}{\partial y_j} = \sum_{i} \left(\vec{1}_{i=j} - \frac{\beta_0 A_{ij}}{\cosh^2(\beta_0 m_i(\vec{y})+\theta_0^{\top}\vec{x}_i)}\right)x_{i,k}.$
	We will bound the term as follows
	\begin{equation}\label{eq:sum3}
	\left|\frac{\partial Q(\vec{y})}{\partial y_j} \right| \leq |x_{j,k}| + \sum_{i\neq j} \beta_0 \abss{A_{ij} x_{i,k}}.
	\end{equation}
	
	Finally let $\vec{y^j} = (\vec{y}_{-j}, -1)$ and note that
	\begin{equation}\label{eq:boundflip2}
	|Q(\vec{y}) - Q(\vec{y^j})| \leq 2 \cdot \left(|x_{j,k}| + \sum_{i\neq j} \beta_0 \abss{A_{ij} x_{i,k}}\right).
	\end{equation}
	We have all the ingredients to complete the proof. We first observe that
	\begin{equation}
	\sum_i \mathbb{E}_{\theta_0,\beta_0}[(y_i - \tanh (\beta_0 m_i(\vec{y})+\theta_0^{\top}\vec{x}_i))Q(\vec{y^i})x_{i,k}]=0,
	\end{equation}
	since
	\begin{equation}
	\begin{array}{cc}
	\mathbb{E}_{\theta_0,\beta_0}[(y_i - \tanh (\beta_0 m_i(\vec{y})+\theta_0^{\top}\vec{x}_i))Q(\vec{y^i})x_{i,k}]= \\= \mathbb{E}_{\theta_0,\beta_0}[\mathbb{E}[(y_i - \tanh (\beta_0 m_i(\vec{y})+\theta_0^{\top}\vec{x}_i))Q(\vec{y^i})x_{i,k}| \vec{y}_{-i}]] =0.
	\end{array}
	\end{equation}
	Therefore it follows
	\begin{align*}
	\mathbb{E}_{\theta_0,\beta_0}[Q^2(\vec{y})] &= \mathbb{E}_{\theta_0,\beta_0}\left[Q(\vec{y}) \cdot \left(\sum_{i} (y_i -\tanh (\beta_0 m_i(\vec{y})+\theta_0^{\top}\vec{x}_i))x_{i,k}\right)\right] \\& = \mathbb{E}_{\theta_0,\beta_0}\left[\sum_{i} \left(Q(\vec{y}) (y_i -\tanh (\beta_0 m_i(\vec{y})+\theta_0^{\top}\vec{x}_i))x_{i,k}\right)\right] \\& = \sum_{i} \mathbb{E}_{\theta_0,\beta_0}\left[(Q(\vec{y}) - Q(\vec{y^i})) \cdot  (y_i -\tanh (\beta_0 m_i(\vec{y})+\theta_0^{\top}\vec{x}_i))x_{i,k}\right] \\& \leq \sum_i 4 \cdot (x_{i,k}^2 + |x_{i,k}|\sum_{j \neq i} \beta_0 |A_{ij}x_{j,k}| ) \\& \leq 4\sum_{i} |x_{i,k}|^2 + B|x_{i,k}| \max_{j} |x_{j,k}|,
	\end{align*}
	and the claim follows by summing over all the coordinates.
\end{proof}

\input{anti}
\subsection{Strong Concavity of Maximum Pseudolikelihood}
\label{sec:logistic-strong-convexity}
In this Section, we set $F = I - X(X^{\top}X)^{-1}X^{\top}$. \footnote{$I-F$ is called hat-matrix or projection matrix.}

\begin{lemma}[Lower Bound on Smallest Eigenvalue of Hessian]
	\label{lem:hessian-lb}
	With probability $1-o(1)$, $\lambda_{\min} ( - H_{(\theta,\beta)}) \geq c$ for some constant $c>0$ for all $(\theta, \beta) \in \mathbb{B}$.
\end{lemma}
\begin{proof}
	We have
	\[
	-H:=G:=
	\left(
	\begin{array}{cc}
	\frac{1}{n}X^{\top}X & \frac{1}{n}X^{\top}\vec{m}\\
	\frac{1}{n}\vec{m}^{\top}X &\frac{1}{n}\norm{\vec{m}}_2^2
	\end{array}
	\right).\]
	Recall our notation that $Q = \frac{1}{n}X^T X$.
	By using the properties of Schur complement, we get that
	\begin{equation}
	\det\left(G-\lambda I\right) = \det \left(Q - \lambda I\right) \det \left(\frac{1}{n}\vec{m}^{\top}\left(I - \frac{1}{n}X\left(Q - \lambda I \right)^{-1}X^{\top}\right)\vec{m}-\lambda\right).
	\end{equation}
	Therefore the minimum eigenvalue of $G$ is at least a positive constant as long as the minimum eigenvalues of \[Q \textrm{ and }\frac{1}{n}\vec{m}^{\top}\left(I - \frac{1}{n}XQ^{-1}X^{\top}\right)\vec{m}\] are positive constants independent of $n$.
	Recall from our assumptions in Table \ref{table:assumptions}, we have that $\lambda_{\min}(Q) \ge c_1$ always where $c_1$ is a positive constant independent of $n$. Hence, it remains to show that \[\lambda_{\min}\left(\frac{1}{n}\vec{m}^{\top}\left(I - \frac{1}{n}XQ^{-1}X^{\top}\right)\vec{m}\right) \ge c_2\] for a positive constant $c_2$ with high probability.
	Recalling that $F =  I - X(X^{\top}X)^{-1}X^{\top}$ our goal is to show that \[ \norm{F\vec{m}}_2^2 \ge c_2 n \; \;\text{ with probability $1-o(1)$}. \]
	Note that $F$ has the property that $F^2 = F$ (i.e. is idempotent) and hence all the eigenvalues of $F$ are $0,1$ (since is of rank $n-d$, it has $d$ eigenvalues zero and $n-d$ eigenvalues one). Moreover, from the sub-multiplicativity of the spectral norm, it holds that
	\begin{align}
	\norm{FA}_2 \leq \norm{F}_2 \times \norm{A}_2 \leq 1. \label{eq:fa2}
	\end{align}
	We also have that $\norm{FA}_{F}^2$ is $\Omega(n)$. This is because $\sigma_i (FA) \geq \sigma_{d+i}(A) \sigma_{n-d-i+1}(F) = \sigma_{d+i}(A)$ for $1\leq i\leq n-d$ ($\sigma_i(G)$ denotes the $i$-th largest eigenvalue of $G$). Since $\sigma_{\max}(A) \leq 1$ it follows that $\norm{FA}_{F}^2 \geq \norm{A}_F^2 - d$.
	
	Below, we provide an important lemma that will be used to show that $\norm{F\vec{m}}_2^2$ is $\Omega(n)$ with high probability.
	\begin{lemma}\label{lem:oneindex} Let $W$ be an $n\times n$ matrix. Fix a pair of indices $i,j$. It holds that
		\[\mathbb{E}_{\theta_0,\beta_0}[(W_i\vec{y})^2 | \vec{y}_{-j}] \geq \frac{e^{-(B+d\cdot M \cdot \Theta)}}{2} W_{ij}^2,\] where $W_i$ is the $i$-th row of $W$.
	\end{lemma}
	\begin{proof}
		For any realization of $\vec{y}$, consider the two summands $\left(\sum_{t\neq j} W_{it}y_t + W_{ij}\right)^2$ and $\left(\sum_{t\neq j} W_{it}y_t - W_{ij}\right)^2$.
		It is clear that if both $\sum_{t\neq j} W_{it}y_t, W_{ij}$ have same sign then the first term is at least $W_{ij}^2$ and if they have opposite sign then the second term is at least $W_{ij}^2$. Additionally, at all times both terms are non-negative.
		
		Moreover from above we have
		\begin{align*}
		\mathbb{E}_{\theta_0,\beta_0}\left[(W\vec{y})_i^2 | \vec{y}_{-j}\right] &\geq W_{i\hat{j}}^2 \cdot \min \left(\Pr\left[y_{j}=+1|\vec{y}_{-j}\right],\Pr\left[y_{j}=-1|\vec{y}_{-j}\right]\right) \\&\geq \frac{W_{ij}^2}{2} \exp\left(-\left|2\theta_0^{\top}\vec{x}_{\hat{j}}+\beta_0 \sum_{t \neq j}A_{jt}y_j\right|\right)\\&\geq \frac{W_{ij}^2}{2} \exp(-\Theta M d - B)
		\end{align*}
		where the inequality before the last holds because of (\ref{eq:costas logistic}) and last inequality by assumption.
	\end{proof}
	
	To proceed, we define an index selection procedure below which will be useful in the later part of the proof.
	\paragraph{An Index Selection Procedure:} Given a matrix $W$, we define $h: [n]\to [n]$ as follows. Consider the following iterative process. At time $t=0$, we start with the $n\times n$ matrix, $W^1 = W$. At time step $t$ we choose from $W^t$ the row with maximum $\ell_2$ norm (let $i_t$ the index of that row, ties broken arbitrarily) and also let $j_t = \textrm{argmax}_j |W^t_{i_tj}| $ (again ties broken arbitrarily). We set $h(i_t) = j_t$ and $W^{t+1}$ is $W^t$ by setting zeros the entries of $i_t^{th}$ row and column $j_t^{th}$. We run the process above for $n$ steps to define the \textbf{bijection} $h$. Below we prove the following lemma.
	\begin{lemma}
		\label{lem:sum-of-maxima-lb}
		Assume that $\norm{FA}_{\infty} \leq c_{\infty}'$ for some positive constant $c_{\infty}'$. We run the process described above on $FA$ and get the function $h$. It holds that
		\[  \sum_i |(FA)_{ih(i)}|^2  \ge \frac{c^4n}{16}.\]
	\end{lemma}
	\begin{proof}
		Let $A^i$ be the $i$-th column of $A$. It is clear that $F \cdot A^i$ corresponds to the $i$-th column of $FA$. Since $F$ has eigenvalues 0,1, we get that $\norm{FA^i}_2^2 \leq \norm{A^i}_2^2 \leq \norm{A^i}_{\infty}\norm{A^i}_{1} \leq 1$ (a) since $A$ is symmetric and hence $\norm{A}_{1} = \norm{A}_{\infty} \leq 1$.
		Above we used the fact that for any $n$-dimensional vector $\vec{u}$, from Holder's inequality,
		\[ \norm{\vec{u}}_2^2 \le \norm{\vec{u}}_{\infty} \norm{\vec{u}}_1. \]
		
		Also let $(FA)_i$ be the $i$-th row of $FA$. Since $\norm{FA}_{\infty} \leq 1$ it holds that $\norm{(FA)_i}_2^2 \leq \norm{FA}_{\infty} \leq 1$ (b).
		
		We run the process described in the previous paragraph on $FA$ and let $W^t$ be the matrix at time $t$ (with $W^1 = FA$). Let $i_1,...,i_n$ be the ordering of the indices of the rows the process chose. It is clear that at every step we remove a column and a row from the matrix, the frobenius norm is decreased by at most two (using facts (a) and (b)).
		
		Hence by the definition of the process we have that
		\[\norm{W^t_{i_t} }_2^2 \geq \frac{\norm{A}_F^2 - 2(t-1)}{n-t-1} \geq \frac{cn - 2t+2}{n} = c - \frac{2(t-1)}{n}.\]
		We set $T = \lfloor cn/2 \rfloor$ and we get
		\begin{align}
		\sum_{t=1}^n \norm{W^t_{i_t}}_2^2 &\geq  \sum_{t=1}^T \norm{W^t_{i_t}}_2^2 \\&\geq
		cT - \frac{2(T-1)T}{2n} \geq cT - \frac{T^2}{n} \approx \frac{c^2 n}{4}
		\end{align}

		Therefore we have that (observe that $1 \geq \norm{(FA)_{j}}_{1} \geq \norm{W^t_{j}}_1$ for all $t,j\in [n]$, i.e., the $\ell_1$ norm of each row does not increase during the process and same is true for $\ell_{\infty}$)
		\begin{align}
		&\frac{c^2 n}{4} \le \sum_{t=1}^n \norm{W^t_{i_t}}_2^2 \le \sum_{i=1}^n \norm{W^t_{i_t}}_{\infty}\norm{W^t_{i_t}}_1 \le  \sum_{i=1}^n \norm{W^t_{i_t}}_{\infty}. \\
		\implies &\frac{c^4n^2}{16} \le \left(\sum_{t=1}^n \norm{W^t_{i_t}}_2^2\right)^2 \le n\sum_{t=1}^n \norm{W^t_{i_t}}_{\infty}^2 \\
		\implies &\sum_{t=1}^n \norm{W^t_{i_t}}_{\infty}^2 \ge \frac{c^4n}{16}.
		\end{align}
		Finally, it holds that $\sum_{t=1}^n \norm{(FA)_t}_{\infty}^2 \geq \sum_{t=1}^n \norm{W^t_{i_t}}_{\infty}^2$ and the claim follows.
	\end{proof}
	
	\begin{corollary}
		\label{cor:cond-expectation}
		$$\sum_{i=1}^n \mathbb{E}_{\theta_0,\beta_0}\left[(F\vec{m})_i^2 | \vec{y}_{-h(i)}\right] \geq Cn,$$ for some positive constant $C$.
	\end{corollary}
	\begin{proof}
		It holds by Lemmas \ref{lem:oneindex} and \ref{lem:sum-of-maxima-lb} by choosing for each $i$, index $h(i)$.
	\end{proof}
	\begin{lemma}[Bounding the variance]
		\label{lem:var-bound-strongconcavity}
		It holds that
		\[\mathbb{E}_{\theta_0,\beta_0}\left[\left(\sum_{i=1}^n (F\vec{m})^2_i - \sum_{i=1}^n \mathbb{E}_{\theta_0,\beta_0}\left[(Fm)_i^2 | \vec{y}_{-h(i)}\right] \right)^2\right] \le 48n+16Bn. \]
	\end{lemma}
	\begin{proof}
		For each $i$, we expand the term $\mathbb{E}_{\theta_0,\beta_0}\left[(F\vec{m})_i^2 | \vec{y}_{-h(i)}\right]$ and we get
		$\mathbb{E}_{\theta_0,\beta_0}\left[(F\vec{m})_i^2 | \vec{y}_{-h(i)}\right] = (\sum_{j \neq h(i)}(FA)_{ij}y_j)^2 + (FA)_{ih(i)}^2 + 2(\sum_{j \neq h(i)}(FA)_{ij}y_j) \tanh(\beta_0 m_{h(i)}(\vec{y})+ \theta_0^{\top}\vec{x}_{h(i)})$. We set $z_{it}(\vec{y}) = 2(\sum_{j \neq t}(FA)_{ij}y_j)$ and we get that the expectation we need to bound is equal to
		\[\mathbb{E}_{\theta_0,\beta_0}\left[\left(\sum_i z_{ih(i)}(\vec{y})y_{h(i)} - z_{ih(i)}(\vec{y})\tanh\left(\beta_0 m_{h(i)}(\vec{y})+ \theta_0^{\top}\vec{x}_{h(i)}\right) \right)^2\right].\]
		First it holds that $\frac{\partial z_{it}}{y_j} = 2(FA)_{ij}$ and $\frac{\partial z_{it}}{y_t} = 0$. Also if $\norm{FA}_{\infty} \leq 1$ it trivially holds that $|z_{it}| \leq 2$. We set \[Q(\vec{y}) = \sum_{i}(y_{h(i)}-\tanh(\beta_0m_{h(i)}(\vec{y})+\theta^{\top}_0 \vec{x}_{h(i)}))z_{ih(i)}(\vec{y}),\]
		hence we get that
		\begin{align}
		&\frac{\partial Q(\vec{y})}{\partial y_j} = \sum_{i} \left(\vec{1}_{h(i)=j} - \frac{\beta_0 A_{h(i)j}}{\cosh^2(\beta_0 m_{h(i)}(\vec{y})+\theta_0^{\top}\vec{x}_{h(i)})}\right)z_{ih(i)}(\vec{y}) \\
		&~~~~~~~~~~~+ \left(y_{h(i)} -\tanh (\beta_0 m_{h(i)}(\vec{y})+\theta_0^{\top}\vec{x}_{h(i)})\right)\frac{\partial z_{ih(i)}(\vec{y})}{\partial y_j}.
		\end{align}
		We will bound the absolute value of each summand. First from above we can bound the second term as follows
		\begin{equation}\label{eq:sum2l}
		\left|(y_{h(i)} -\tanh (\beta_0 m_{h(i)}(\vec{y})+\theta_0^{\top}\vec{x}_{h(i)}))\frac{\partial z_{ih(i)}(\vec{y})}{\partial y_j}\right| \leq 4 \abss{(FA)_{ij}}.
		\end{equation}
		Using the fact that $\frac{1}{\cosh^2(x)} \leq 1$ it also follows that
		\begin{equation}\label{eq:sum1l}
		\left|\sum_{i} \left(\vec{1}_{h(i)=j} - \frac{\beta_0 A_{h(i)j}}{\cosh^2(\beta_0 m_{h(i)}(\vec{y})+\theta_0^{\top}\vec{x}_{h(i)})}\right)z_{ih(i)}(\vec{y})\right| \leq \sum_{i} \vec{1}_{h(i)=j}|z_{ih(i)}(\vec{y})| + \sum_{i\neq j} \abss{\beta_0 A_{h(i)j} z_{ih(i)}(\vec{y})},
		\end{equation}
		which is at most 2$\sum_{i} \vec{1}_{h(i)=j} + 2B$.
		
		Using (\ref{eq:sum2l}) and  (\ref{eq:sum1l}) it follows that $\left|\frac{\partial Q(\vec{y})}{\partial y_j}\right|  \leq 4\sum_{i, h(i)\neq j} \abss{(FA)_{ij}} +2 \sum_{i} \vec{1}_{h(i)=j} + 2B$.
		We have all the ingredients to complete the proof. Let $\vec{y^j} = (\vec{y}_{-j}, -1)$. We first observe that
		\begin{equation}
		\sum_i \mathbb{E}_{\theta_0,\beta_0}[(y_{h(i)} - \tanh (\beta_0 m_{h(i)}(\vec{y})+\theta_0^{\top}\vec{x}_{h(i)}))Q(\vec{y^{h(i)}})z_{ih(i)}(\vec{y})]=0,
		\end{equation}
		since
		\begin{equation}
		\begin{array}{cc}
		\mathbb{E}_{\theta_0,\beta_0}[(y_{h(i)} - \tanh (\beta_0 m_{h(i)}(\vec{y})+\theta_0^{\top}\vec{x}_{h(i)}))Q(\vec{y^{h(i)}})z_{ih(i)}(\vec{y})]= \\= \mathbb{E}_{\theta_0,\beta_0}[\mathbb{E}[(y_{h(i)} - \tanh (\beta_0 m_{h(i)}(\vec{y})+\theta_0^{\top}\vec{x}_{h(i)}))Q(\vec{y^{h(i)}})z_{ih(i)}(\vec{y})| \vec{y}_{-h(i)}]] =0.
		\end{array}
		\end{equation}
		Therefore it follows
		\begin{align*}
		\mathbb{E}_{\theta_0,\beta_0}[Q^2(\vec{y})] &= \mathbb{E}_{\theta_0,\beta_0}\left[Q(\vec{y}) \cdot \left(\sum_{i} (y_{h(i)} -\tanh (\beta_0 m_i(\vec{y})+\theta_0^{\top}\vec{x}_i))z_{ih(i)}(\vec{y})\right)\right] \\& = \mathbb{E}_{\theta_0,\beta_0}\left[\sum_{i} \left(Q(\vec{y}) (y_{h(i)} -\tanh (\beta_0 m_i(\vec{y})+\theta_0^{\top}\vec{x}_i))z_{ih(i)}(\vec{y})\right)\right] \\& = \sum_{i} \mathbb{E}_{\theta_0,\beta_0}\left[(Q(\vec{y}) - Q(\vec{y^{h(i)}})) \cdot  (y_{h(i)} -\tanh (\beta_0 m_{h(i)}(\vec{y})+\theta_0^{\top}\vec{x}_{h(i)}))z_{ih(i)}(\vec{y})\right] \\&
		= \sum_{i} \mathbb{E}_{\theta_0,\beta_0}\left[(Q(\vec{y}) - Q(\vec{y^{i}})) \cdot  (y_{i} -\tanh (\beta_0 m_i(\vec{y})+\theta_0^{\top}\vec{x}_i))z_{h^{-1}(i)i}(\vec{y})\right] \\&
		\leq 8\sum_i \left[\sum_{t,h(t)\neq i} 4|(FA)_{ti}|+2\sum_t \vec{1}_{h(t)=i}+ 2B\right]\\&
		\leq    48n+16Bn,
		\end{align*}
		where we also used the fact that $\sum_{i=1}^n \sum_j (FA)_{ij} \leq \sum_i \norm{FA}_{\infty} \leq n$.
	\end{proof}
	\begin{remark} The proof of Lemma \ref{lem:sum-of-maxima-lb} depends on the assumption that $\norm{FA}_\infty \leq c'_{\infty}$ for some positive constant $c_{\infty}'$. This assumption is not necessary for the proof to go through and below we will argue about removing this assumption.
	\end{remark}
	\begin{lemma}\label{lem:sumsmall} It holds that $\sum_{i,j} |F_{ij}| \leq n(d+1)$.
	\end{lemma}
	\begin{proof}
		Since $F$ has rank $n-d$ and has eigenvalues 0,1 it holds that $F = I - \sum_{i=1}^d \vec{v}^i\vec{v}^{i\; \top}$ with $\norm{\vec{v}^i}_2 =1$. We choose vector $\vec{x}$ so that $x_j=1$ if $v^i_j\geq 0$ and $x_j=-1$ otherwise. It holds that $(\sum_{j} |v^i_{j}|)^2  = \vec{x}^{\top} \vec{v}^i \vec{v}^{i \; \top} \vec{x} \leq \norm{\vec{x}}^2_2 = n$. Thus we get that $\sum_{j}\sum_t |v^i_j| \cdot |v^i_t| \leq n$. Thefore we get that \[\sum_{i,j} |F_{i,j}| \leq n+ \sum_{i=1}^d n = n(d+1).\]
	\end{proof}
	\begin{corollary}
		\label{cor:fa_inftynorm}
		We can drop the assumption $\norm{FA}_{\infty}$ is $O(1)$.
	\end{corollary}
	\begin{proof} From Lemma \ref{lem:sumsmall} we get that for every $\epsilon>0$, there exists at most $nd/\epsilon$ rows of $FA$ that have $\ell_1$ norm larger than $\epsilon$ (by Markov's inequality), hence there exist $n(1-d/\epsilon)$ rows that have $\ell_1$ norm at most $\epsilon$. Moreover, since $\norm{FA}_{2} \leq 1$ \eqref{eq:fa2} it holds that each row $j$ of $FA$ (denote by $(FA)_j$) has $\norm{(FA)_j}_2^2 \leq 1$. The trick is to remove from matrix $FA$ all the rows that have $\ell_1$ norm larger than $\epsilon$ and we reduce the Frobenius norm squared of $FA$ by $nd/\epsilon$. By choosing $\epsilon$ to be $\frac{2nd}{\norm{FA}_F^2}$, the resulting matrix (after removing the ``overloaded" rows) has still Frobenius norm squared $\Omega(n)$ and moreover all the rows have $\ell_1$ norm at most $\epsilon$ (which is $\Theta(d)$). The claim follows since the rows we are removing cannot increase the expression $\sum_i (FA\vec{y})^2_i$ (sum of squares). This gives $d^2\sqrt{\frac{d}{n}}$ consistency (without the assumption $\norm{FA}_{\infty}$ is $O(1)$) and when the dot product between $\theta$ and feature vectors is $O(1)$.
	\end{proof}

Finally we are ready to complete the proof by showing that $\norm{F\vec{m}}_2^2 \ge c_2n$ with probability $\ge 1-\delta$ for some $\delta \le 1/3$ and a small enough constant $c_2$. From Lemma \ref{lem:var-bound-strongconcavity} and Markov's inequality, we have
\begin{align}
&\Pr\left[ \left(\sum_{i=1}^n (F\vec{m})^2_i - \sum_{i=1}^n \mathbb{E}_{\theta_0,\beta_0}\left[(Fm)_i^2 | \vec{y}_{-h(i)}\right] \right)^2 \ge n^{1.2}\right] \le \frac{48+16B}{n^{0.2}} \\
\implies &\Pr\left[ \sum_{i=1}^n (F\vec{m})^2_i  \le Cn - n^{0.6}\right] \le \frac{48+16B}{n^{0.2}} \label{eq:sc11}\\
\implies &\Pr\left[ \sum_{i=1}^n (F\vec{m})^2_i  \ge Cn/2\right] \ge 1-o(1),
\end{align}
where in \eqref{eq:sc11} we used that $\sum_{i=1}^n \mathbb{E}_{\theta_0,\beta_0}\left[(F\vec{m})_i^2 | \vec{y}_{-h(i)}\right] \geq Cn$ (Corollary \ref{cor:cond-expectation}).

\end{proof}

\subsection{Completing the Proof}
\label{sec:logistic-completing-proof}
With the above results in hand, we can prove the main result of this Section, Theorem~\ref{thm:logistic}.\\
\begin{prevproof}{Theorem}{thm:logistic}
	First it holds that:
	\begin{align}
	& \mathbb{E}_{\theta_0,\beta_0}[\norm{\nabla LPL(\theta_0,\beta_0) }^2_{2}]\\
	&= \sum_{k=1}^d\mathbb{E}_{\theta_0,\beta_0} \left[\frac{1}{n}\sum_{i=1}^n \left(x_{i,k}y_i - x_{i,k}\tanh (\beta_0 m_i(\vec{y})+\theta_0^{\top}\vec{x}_i) \right)\right]^2 \label{eq:logistic-grad-bound1}\\
	& \: \: + \mathbb{E}_{\theta_0,\beta_0} \left[\frac{1}{n}\sum_{i=1}^n \left(y_i m_i(\vec{y})- m_i(\vec{y})\tanh (\beta_0 m_i(\vec{y})+\theta_0^{\top}\vec{x}_i) \right)\right]^2 \label{eq:logistic-grad-bound2}.
	\end{align}
	
	From Lemmas \ref{lem:conc1} and \ref{lem:conc2} we have that
	\begin{align}
	\mathbb{E}_{\theta_0,\beta_0}[\norm{\nabla LPL(\theta_0,\beta_0) }^2_{2}] \leq \frac{c}{n}
	\end{align}
	for some constant $c$. By Markov's inequality,
	we get that
	\begin{align}
	\Pr \left[\norm{\nabla LPL(\theta_0,\beta_0) }^2_{2} \le  \frac{c}{n\delta}  \right] \ge 1- \delta.
	\end{align}
	for any constant $\delta$.
	Next, we have from Lemma~\ref{lem:hessian-lb} that, $\min_{(\theta, \beta) \in \mathbb{B}} \lambda_{\min}\left(-H_{(\theta,\beta)}\right) \ge C$ for some constant $C$ independent of $n$.
	Plugging into Lemma~\ref{lem:logistic-consistency}, we get that
	\begin{align}
	\norm{(\theta_{\mathcal{D}} - \theta_0,\beta_{\mathcal{D}}-\beta_0)}_2  = \mathcal{D} \norm{(\hat{\theta}_{MPL} - \theta_0,\hat{\beta}_{MPL}-\beta_0)}_2 \le \frac{\norm{\nabla LPL(\theta_0,\beta_0) }_2}{\min_{(\theta, \beta) \in \mathbb{B}} \lambda_{\min}\left(-H_{(\theta,\beta)}\right)} \label{eq:lc5}
	\end{align}
	Now we have from the above that $\norm{(\theta_{\mathcal{D}} - \theta_0,\beta_{\mathcal{D}}-\beta_0)}_2 \to 0$ as $n \to \infty$ (is of order $\frac{1}{\sqrt{n}}$). Also note that any point on the boundary of $\mathbb{B}$ has a fixed distance to $(\theta_0,\beta_0)$ since it lies in the interior. Hence $\norm{(\theta_{\mathcal{D}} - \theta_0,\beta_{\mathcal{D}}-\beta_0)}_2 \to 0$ implies that $\mathcal{D} \to 1$ which in turn implies that $\mathcal{D} \ge 1/2$ for sufficiently large $n$. Hence
	\begin{align}
	 \eqref{eq:lc5} \implies \norm{(\hat{\theta}_{MPL} - \theta_0,\hat{\beta}_{MPL}-\beta_0)}_2 &\le \frac{2\norm{\nabla LPL(\theta_0,\beta_0) }_2}{\min_{(\theta, \beta) \in \mathbb{B}} \lambda_{\min}\left(-H_{(\theta,\beta)}\right)} \\
	 &\le O_d\left(\frac{1}{\sqrt{n}}\right)
	\end{align}
	with probability $\ge 1-\delta$.
\end{prevproof}

%

\section{Linear Regression with Dependent Observations}\label{sec:linear}
In this section we focus on linear regression under weakly dependent errors. As opposed to Logistic regression, in linear regression the log-likelihood is computationally tractable.
\subsection{Our Model}
We recall the model for dependent observations we consider.
\noindent In our setting, we have that the errors $\epsilon_i = y_i - \theta x_i$ are distributed according to a Gaussian graphical model. Since each $\epsilon_i$ is zero mean, we have that $\mu = 0$ in our case.
Also, similar to the logistic regression setting, we will assume that we have complete knowledge of the graph structure up to a scaling factor. That is, $(\Sigma)^{-1} = \beta A + D$ where the matrix $A$ is a known symmetric matrix with $A_{ii} = 0$ and $D = [ d_1 \ldots d_n]^T$ is a known diagonal matrix with positive entries. Hence the probability distribution of the observations is:
\begin{align}
\Pr[\vec{y} = \vec{a}] = \frac{\exp\left(- \frac{1}{2}(\vec{a}-\theta^{\top}\vec{x})^{\top} \Sigma^{-1} (\vec{a} - \theta^{\top}\vec{x}) \right)}{(2\pi)^{n/2} \det(\Sigma)^{1/2}}
\end{align}

This section is devoted to showing that the Maximum Likelihood Estimator (MLE) under appropriate (over)re-parametrization - (the new parameter vector will be $(\theta,\beta,\kappa)$ where $\theta, \beta$ remain same after reparametrization and $\kappa = \beta.\theta$) - is $\sqrt{n}$ consistent in our linear regression model. We set $\mathbb{B} = [-\Theta,\Theta]^d \times \left[-B , B\right] \times \left[- \Theta B,  \Theta B\right]^d$ (the set of which the parameters should be in the interior and $B$ is defined in Table \ref{table:assumptions}). Formally we prove the following theorem.

\begin{theorem}[Main Linear]
	\label{thm:linear}
Assume that
\begin{enumerate}
\item Feature matrix $X$\footnote{$\vec{x}_i$ can be subgaussian.} with covariance matrix $Q = \frac{1}{n} X^{\top}X$ having $\lambda_{\min}(Q), \lambda_{\max}(Q)$ as positive constants.
\item $\norm{A}_{2}$ is $\Theta(1)$, $\norm{A}^2_F$ is $\Omega(n)$ and $\lambda_{\min}\left(\frac{1}{n} X^{\top}A^{\top}(I - DX(X^{\top}D^2X)^{-1}X^{\top}D) AX\right)$ is $\Theta(1)$.
\end{enumerate} 	
We show that the Maximum Log-likelihood Estimate (MLE) $(\hat{\theta},\hat{\beta},\hat{\kappa})$ is $O_d\left({\sqrt{1 \over n}}\right)$ consistent as long as the true parameter vector $(\theta_0,\beta_0, \kappa_0) \in \mathbb{B}$ (in the interior of $\mathbb{B}$), i.e., for each $\delta>0$ and $n$ sufficiently large,
	$\norm{(\hat{\theta},\hat{\beta},\hat{\kappa}) - (\theta_0,\beta_0, \kappa_0)}_2$ is $O_d\left({\sqrt{1 \over n}}\right)$ with probability $1-\delta$. Moreover, we can compute a vector $(\tilde{\theta},\tilde{\beta})$ with $\norm{(\hat{\theta},\hat{\beta}) - (\tilde{\theta},\tilde{\beta})}_2$ to be $O_d\left({\sqrt{1 \over n}} \right)$ in $O(\ln n)$ iterations of projected gradient descent\footnote{Each iteration is polynomial time computable} with probability $1-o(1)$.
\end{theorem}

\begin{corollary}[Application to Sherrington-Kirkpatrick (S–K) model]\label{cor:sk} In the Sherrington-Kirkpatrick model \cite{chatterjee24}, we have that $A_{ij} = \frac{g_{ij}}{\sqrt{n}}$ for $i<j$, where $g_{ij} \sim \mathcal{N}(0,1)$ and $A_{ji} = A_{ij}, A_{ii}=0$. From Lemma \ref{lem:SKmodel}, it follows that $A$ satisfies the assumptions of our main theorem, so we can infer $\beta, \theta$ (with a $\sqrt{n}$ rate of consistency).
\end{corollary}

Technically, we will reparametrize the log-likelihood function in such a way that the new parameter vector is not high-dimensional and the resulting log-likelihood becomes strongly convex. The reparametrization and the equations of log-likelihood, its gradient and Hessian can be found in Section \ref{sec:reparameter}. We will follow the same high level ideas as in the logistic regression. Under assumptions on $A, \beta,\theta, (\beta A+D)^{-1}, AX$ that are summarized in Table \ref{table:assumptions} we proceed as follows:
\begin{itemize}
\item We prove concentration results for the gradient of the reparametrized log-likelihood, see Lemmas
\ref{lem:boundgrbeta},\ref{lem:boundgrtheta} and \ref{lem:boundgrkappa} in Section \ref{sec:concentrationlinear}.
\item We prove that the minimum eigenvalue of the negative Hessian of reparametrized log-likelihood is large enough, see Lemma \ref{lem:boundeigenvaluelinear} in Section \ref{sec:lowerhessianlinear}.
\end{itemize}

Below we provide some important definitions.
\subsection{Our Reparametrization and Log-likelihood}\label{sec:reparameter}
It is not hard to see that the negative log-likelihood is not convex with respect to the parameter vector $(\theta,\beta) \in \mathbb{R}^{d+1}$, for the linear regression model with dependent errors. Nevertheless, we can reparametrize the log-likelihood in such a way to make it convex. The classic way to do it sets $T := \Sigma^{-1}, \nu = \Sigma^{-1}\mu$ (for a gaussian $\mathcal{N}(\mu,\Sigma)$). However, this creates a parameter vector of dimension $\Omega(n)$. It is crucial that after the reparametrization the dimensionality of the new parameter vector is $O(d)$ and not $\Omega(n)$ (for concentration purposes, see remark \ref{rem:good}). Hence, we take a different route here.

We set $\kappa := \beta \cdot \theta \in \mathbb{R}^d$ and define our parametric vector to be $(\theta,\beta,\kappa) \in \mathbb{R}^{2d+1}$ (parameters $\theta,\beta$ remain the same and we introduce vector $\kappa$). The vector $(\theta,\beta,\kappa)$ is $(2d+1)$-dimensional. Our reparameterization is an overparameterization which helps us achieve convexity of the negative log-likelihood function.

The negative log-likelihood is given by the following:
\begin{align}
-LL &= \frac{1}{2}(\vec{y} - X\theta)^{\top}(\beta A+D)(\vec{y}-X\theta) + \log \int_{\mathbb{R}^n} \exp(-\frac{1}{2}(\vec{z}-X\theta)^{\top}(\beta A+D)(\vec{z}-X\theta))dz\\
& =  \frac{1}{2}\vec{y}^{\top}(\beta A+D)\vec{y} - \vec{y}^{\top}AX\kappa - \vec{y}^{\top}DX\theta + \log \int_{\mathbb{R}^n} \exp(-\frac{1}{2}\vec{z}^{\top}(\beta A+D)\vec{z}+\vec{z}^{\top}AX\kappa + \vec{z}^{\top}DX\theta)dz
\end{align}
The negative gradient of the log-likelihood is given below:
\begin{align}
-\nabla_{\theta} LL(\theta,\beta, \kappa) &= -\vec{y}^{\top}DX + \frac{\int_{\mathbb{R}^n} \vec{z}^{\top}DX\exp(-\frac{1}{2}\vec{z}^{\top}(\beta A+D)\vec{z}+\vec{z}^{\top}AX\kappa + \vec{z}^{\top}DX\theta)dz}{\int_{\mathbb{R}^n} \exp(-\frac{1}{2}\vec{z}^{\top}(\beta A+D)\vec{z}+\vec{z}^{\top}AX\kappa + \vec{z}^{\top}DX\theta)dz}\\& = -\vec{y}^{\top}DX + E_{\vec{z} \sim \mathcal{N}((\beta A+D)^{-1}(AX\kappa+DX\theta),(\beta A+D)^{-1})}\left[\vec{z}^{\top}\right]DX \\
-\nabla_{\beta} LL(\theta,\beta, \kappa) &= \frac{1}{2}\vec{y}^{\top}A\vec{y} + \frac{\int_{\mathbb{R}^n} -\frac{1}{2}\vec{z}^{\top}A\vec{z}\exp(-\frac{1}{2}\vec{z}^{\top}(\beta A+D)\vec{z}+\vec{z}^{\top}AX\kappa + \vec{z}^{\top}DX\theta)dz}{\int_{\mathbb{R}^n} \exp(-\frac{1}{2}\vec{z}^{\top}(\beta A+D)\vec{z}+\vec{z}^{\top}AX\kappa + \vec{z}^{\top}DX\theta)dz}\\&=
\frac{1}{2}\vec{y}^{\top}A\vec{y} - E_{\vec{z} \sim \mathcal{N}((\beta A+D)^{-1}(AX\kappa+DX\theta),(\beta A+D)^{-1})}\left[\frac{1}{2}\vec{z}^{\top}A\vec{z}\right]
\\
-\nabla_{\kappa} LL(\theta,\beta, \kappa) &= -\vec{y}^{\top}AX + \frac{\int_{\mathbb{R}^n} \vec{z}^{\top}AX\exp(-\frac{1}{2}\vec{z}^{\top}(\beta A+D)\vec{z}+\vec{z}^{\top}AX\kappa + \vec{z}^{\top}DX\theta)dz}{\int_{\mathbb{R}^n} \exp(-\frac{1}{2}\vec{z}^{\top}(\beta A+D)\vec{z}+\vec{z}^{\top}AX\kappa + \vec{z}^{\top}DX\theta)dz}\\&=-\vec{y}^{\top}AX + E_{\vec{z} \sim \mathcal{N}((\beta A+D)^{-1}(AX\kappa+DX\theta),(\beta A+D)^{-1})}\left[\vec{z}^{\top}\right]AX
.
\end{align}
The negative hessian of the log-likelihood is given below (it is of size $(2d+1)\times (2d+1)$):
\begin{equation}\label{eq:Hessianlinear}
-H:= -\nabla^2 LL = Cov_{\vec{z} \sim \mathcal{N}((\beta A+D)^{-1}(AX\kappa+DX\theta),(\beta A+D)^{-1})}\left[
\left(
\begin{array}{ccc}
-\frac{1}{2}\vec{z}^{\top}A\vec{z}\\
X^{\top}D\vec{z}\\
X^{\top}A\vec{z}
\end{array}
\right),
\left(
\begin{array}{ccc}
-\frac{1}{2}\vec{z}^{\top}A\vec{z}\\
X^{\top}D\vec{z}\\
X^{\top}A\vec{z}
\end{array}
\right)
\right].
\end{equation}
\subsection{Consistency of Likelihood}
\label{sec:linear-consistency}
Let $(\theta_0,\beta_0, \kappa_0)$ be the true parameter (observe that $\kappa_0 = \beta_0 \cdot \theta_0$). We define $(\theta_t,\beta_t,\kappa_t) = (1-t)(\theta_0,\beta_0,\kappa_0)+ t(\hat{\theta},\hat{\beta},\hat{\kappa})$ where $(\hat{\theta},\hat{\beta},\hat{\kappa})$ satisfies the first order conditions for $LL$ (i.e., $\nabla LL(\hat{\theta},\hat{\beta},\hat{\kappa}) = \vec{0}$) \footnote{Observe that is \textit{not} necessarily true that $\hat{\kappa} = \hat{\beta} \hat{\theta}$.} and set
\[g(t) := (\theta_0 - \hat{\theta},\beta_0 - \hat{\beta},\kappa_0-\hat{\kappa})^{\top} \nabla LL(\theta_t,\beta_t,\kappa_t),\; g'(t) = -(\theta_0 - \hat{\theta},\beta_0 - \hat{\beta},\kappa_0-\hat{\kappa})^{\top} H_{(\theta_t,\beta_t,\kappa_t)}(\theta_0 - \hat{\theta},\beta_0 - \hat{\beta},\kappa_0 - \hat{\kappa}).\]
Let $\mathcal{D} \in [0,1]$ be such that $(\theta_{\mathcal{D}}, \beta_{\mathcal{D}},\kappa_{\mathcal{D}})$ intersects the boundary of set $\mathbb{B}$ (if it does not intersect the boundary of $\mathbb{B}$ then $\mathcal{D}=1$). Since $H$ is negative semidefinite (from analysis in equation \ref{eq:Hessianlinear}) we have that $g'(t) \geq 0$ (**). It holds that
\begin{align*}
\norm{(\theta_0 - \hat{\theta},\beta_0 - \hat{\beta},\kappa_0 - \hat{\kappa})}_2 \cdot \norm{\nabla LL(\theta_0,\beta_0,\kappa_0)}_2 &\geq |(\theta_0 - \hat{\theta},\beta_0 - \hat{\beta},\kappa_0-\hat{\kappa})^{\top} \nabla LL(\theta_0,\beta_0,\kappa_0)|  \\&=
|g(1)-g(0)| \\&= \left|\int_{0}^1 g'(t)dt\right| \\&\geq \left|\int_{0}^{\mathcal{D}} g'(t)dt\right| \textrm{ by (**)}\\& \geq \mathcal{D}\min_{(\theta,\beta,\kappa) \in \tilde{\mathbb{B}}}\lambda_{\min}\left(-H_{(\theta, \beta,\kappa)}\right) \norm{(\theta_0 - \hat{\theta},\beta_0 - \hat{\beta},\kappa_0-\hat{\kappa})}^2_2.
\end{align*}

The aforementioned inequalities indicate that we need a concentration result for $\norm{\frac{1}{n}\nabla LL(\theta_0,\beta_0,\kappa_0)}_2$ and a lower bound on the minimum eigenvalue of $-\frac{1}{n}H$ for consistency of the MLE. As in the logistic regression case, combining with the observation that $D \to 1$\footnote{This is true because $\norm{(\theta_{\mathcal{D}} - \theta_0,\beta_{\mathcal{D}}-\beta_0, \kappa_{\mathcal{D}} - \kappa_0) }_2 \to 0$ as $n \to \infty$ by showing the promised concentration result and the lower bound).} as $n \to \infty$ (i.e., $D \geq \frac{1}{2}$ for $n$ sufficiently large) we get the desired rate of consistency.

\subsubsection{Concentration results}\label{sec:concentrationlinear}
We have
\begin{equation}\label{eq:gradientLL}
\begin{array}{cc}
\mathbb{E}_{\theta_0,\beta_0}\left[\norm{\nabla LL(\theta_0,\beta_0, \beta_0 \cdot \theta_0)}_2^2 \right] =
\mathbb{E}_{\theta_0,\beta_0}\left[\norm{\nabla_{\theta} LL(\theta_0,\beta_0, \beta_0 \cdot \theta_0)}_2^2 \right] + \\ \mathbb{E}_{\theta_0,\beta_0}\left[|\nabla_{\beta} LL(\theta_0,\beta_0, \beta_0 \cdot \theta_0)|^2 \right]+ \mathbb{E}_{\theta_0,\beta_0}\left[\norm{\nabla_{\kappa} LL(\theta_0,\beta_0, \beta_0 \cdot \theta_0)}_2^2 \right]
\end{array}
\end{equation}
We prove below concentration results for each term separately.
\begin{lemma}[Bounding the 1st term]\label{lem:boundgrtheta} \[\mathbb{E}_{\theta_0,\beta_0}\left[\norm{\nabla_\theta LL(\theta_0,\beta_0,\kappa_0)}^2_2\right] = \norm{(\beta_0A+D)^{-1/2}DX}_F^2 \leq \norm{(\beta_0A+D)^{-1/2}D}_2^2 \norm{X}^2_F.\]
\end{lemma}
\begin{proof}
	Assume that $y \sim \mathcal{N}(X\theta_0,(\beta_0A+D)^{-1})$, it follows that $-\nabla_\theta LL(\theta_0,\beta_0,\kappa_0) = -\vec{y}^{\top}DX - \mathbb{E}_{\theta_0,\beta_0}[-\vec{y}^{\top}DX].$ It is clear that the vector $\vec{w} = (\beta_0A+D)^{1/2}(\vec{y}- X\theta_0 ) \sim \mathcal{N}(\vec{0},\vec{I}).$
	
	It follows that
	\begin{align}
	-\vec{y}^{\top}DX - \mathbb{E}_{\theta_0,\beta_0}[-\vec{y}^{\top}DX] &= -((\beta_0A+D)^{-1/2}\vec{w}+X\theta_0)^{\top}DX + \mathbb{E}_{\theta_0,\beta_0}[((\beta_0A+D)^{-1/2}\vec{w}+X\theta_0)^{\top}DX]\\&=
	-\vec{w}^{\top}(\beta_0A+D)^{-1/2}DX.
	\end{align}
	It holds that
	\begin{align*}
	\mathbb{E}_{\theta_0,\beta_0}[\norm{\nabla_\theta LL(\theta_0,\beta_0,\kappa_0)}^2_2] &= \mathbb{E}_{\theta_0,\beta_0}\left[\norm{\vec{w}^{\top}(\beta_0A+D)^{-1/2}DX}_2^2\right] \\& = \mathrm{tr}((\beta_0A+D)^{-1/2}DXX^{\top}D(\beta_0A+D)^{-1/2})\\& =\norm{(\beta_0A+D)^{-1/2}DX}_F^2
	\end{align*}
\end{proof}

\begin{lemma}[Bounding the 3rd term]\label{lem:boundgrkappa} Similarly to Lemma \ref{lem:boundgrtheta} we get \[\mathbb{E}_{\theta_0,\beta_0}\left[\norm{\nabla_\kappa LL(\theta_0,\beta_0,\kappa_0)}^2_2\right] = \norm{(\beta_0A+D)^{-1/2}AX}_F^2 \leq \norm{(\beta_0A+D)^{-1/2}A}_2^2 \norm{X}^2_F.\]
\end{lemma}

\begin{lemma}[Bounding the 2nd term] \label{lem:boundgrbeta} It holds that
\begin{equation}
\begin{array}{cc}
\mathbb{E}_{\theta_0,\beta_0,\kappa_0}[|\nabla_{\beta}LL (\theta_0,\beta_0,\kappa_0)|^2] = \mathbb{V}_{\vec{z} \sim \mathcal{N}(X\theta_0,(\beta_0A+D)^{-1})}[\vec{z}^{\top}A\vec{z}]\\ \leq 2 \norm{(\beta_0A+D)^{-1/2}A(\beta_0A+D)^{-1/2}}^2_F + 4d\Theta^2 \norm{(\beta_0A+D)^{-1/2}AX}^2_2
\end{array}
\end{equation}
\end{lemma}
\begin{proof}
	We follow the calculations of Lemma \ref{lem:quadratic}. It holds that
	\begin{align*}
	\mathbb{E}_{\theta_0,\beta_0,\kappa_0}[|\nabla_{\beta}LL (\theta_0,\beta_0,\kappa_0)|^2] &\leq 2 \mathrm{tr}((A(\beta_0A+D)^{-1})^2)+ 4d\Theta^2 \norm{(\beta_0A+D)^{-1/2}AX}^2_2 \\ &\leq 2 \mathrm{tr}((\beta_0A+D)^{-1/2}A(\beta_0A+D)^{-1/2}(\beta_0A+D)^{-1/2}A(\beta_0A+D)^{-1/2})\\&\;\;\;\;\;\;\;\;+ 4d\Theta^2 \norm{(\beta_0A+D)^{-1/2}AX}^2_2
	\\ &= 2 \norm{(\beta_0A+D)^{-1/2}A(\beta_0A+D)^{-1/2}}^2_F + 4d\Theta^2 \norm{(\beta_0A+D)^{-1/2}AX}^2_2
	\end{align*}
\end{proof}

\begin{remark}\label{rem:good} We note the dependence on the dimension $d$ for the bound in Lemma \ref{lem:boundgrbeta}. This indicates how crucial it is that the dimensionality of the parameter vector does not scale with $n$.
\end{remark}
\subsubsection{Lower bound on the minimum eigenvalue}\label{sec:lowerhessianlinear}
In this section we provide a lower bound on the minimum eigenvalue of the negative Hessian of the log-likelihood. We need this bound for strong concavity of the log-likelihood.
\begin{lemma}[Bounding the minimum eigenvalue]\label{lem:boundeigenvaluelinear} Let $z \sim \mathcal{N}((\beta A+D)^{-1}(AX\kappa+DX\theta),(\beta A+D)^{-1})$. There exists a constant $C$ such that
\begin{equation}
\begin{array}{cc}
\lambda_{\min}\left(Cov\left[
\left(
\begin{array}{ccc}
-\frac{1}{2}\vec{z}^{\top}A\vec{z}\\
X^{\top}D\vec{z}\\
X^{\top}A\vec{z}
\end{array}
\right),
\left(
\begin{array}{ccc}
-\frac{1}{2}\vec{z}^{\top}A\vec{z}\\
X^{\top}D\vec{z}\\
X^{\top}A\vec{z}
\end{array}
\right)
\right]\right) \geq \frac{Cn}{d}.
\end{array}
\end{equation}
\end{lemma}
\begin{proof}
	We set $\mu = (\beta A+D)^{-1}(AX\kappa+DX\theta)$, $\Sigma = (\beta A+D)^{-1}$ and $\vec{w} =\Sigma^{-1/2}(\vec{z}- \mu)$ (i.e., $\vec{w} \sim \mathcal{N}(\vec{0},\vec{I})$) and we consider the vector
	\[
	\vec{h} :=
	\left(
	\begin{array}{ccc}
	-\frac{1}{2}\vec{w}^{\top}\Sigma^{1/2}A\Sigma^{1/2}\vec{w} + \vec{w}^{\top}\Sigma^{1/2}A\mu - \mathbb{E}[-\frac{1}{2}\vec{w}^{\top}\Sigma^{1/2}A\Sigma^{1/2}\vec{w}]\\
	X^{\top}D\Sigma^{1/2}\vec{w}\\
	X^{\top}A\Sigma^{1/2}\vec{w}
	\end{array}
	\right).\]
	Let $\vec{v} := (v_1,\vec{v}_2,\vec{v}_3)$ be a column vector where $v_1 \in \mathbb{R}$, $\vec{v}_2, \vec{v}_3 \in \mathbb{R}^d$ so that $\norm{\vec{v}}_2=1$. It follows that $-\vec{v}^{\top}H\vec{v} = \mathbb{E}[(\vec{v}^{\top}\vec{h})^2]$ where $-H$ is the negative hessian computed in (\ref{eq:Hessianlinear}).
	
	From Lemma \ref{lem:quadratic} we get that
	
	\begin{align*}
	\mathbb{E}\left[(\vec{v}^{\top}\vec{h})^2\right] &= v_1^2 \left(\frac{1}{2}\mathrm{tr}\left((\Sigma^{1/2}A\Sigma^{1/2})^2\right)\right) + \norm{v_1\mu^{\top}A \Sigma^{1/2}+ \vec{v}_2^{\top}X^{\top}D\Sigma^{1/2}+\vec{v}_3^{\top}X^{\top}A\Sigma^{1/2}}_2^2.
	\end{align*}
	If a positive constant mass at least $\sqrt{\frac{\epsilon}{d}}$ is put on $v_1$ then the above term is at least $\frac{\epsilon}{2d} \norm{\Sigma^{1/2}A\Sigma^{1/2}}^2_F$ which is $\Theta(\frac{n}{d})$. If not then the above term is at least $\norm{ \vec{v}_2^{\top}X^{\top}D\Sigma^{1/2}+\vec{v}_3^{\top}X^{\top}A\Sigma^{1/2}}_2^2 - O(n\sqrt{\epsilon})$ with $\norm{v_2}_2^2+\norm{v_3}_2^2 \geq 1 - \epsilon$. We will prove a $\Theta(n)$ lower bound on the term $\norm{ \vec{v}_2^{\top}X^{\top}D\Sigma^{1/2}+\vec{v}_3^{\top}X^{\top}A\Sigma^{1/2}}_2^2 \geq \norm{ \vec{v}_2^{\top}X^{\top}D+\vec{v}_3^{\top}X^{\top}A}_2^2 \sigma_{\min}(\Sigma)$.
	
	It suffices to bound the minimum eigenvalue of the following matrix:
	\[
	M:=
	\left(
	\begin{array}{cc}
	\frac{1}{n}X^{\top}D^2 X & \frac{1}{n}X^{\top}DAX\\
	\frac{1}{n}X^{\top}ADX & \frac{1}{n}X^{\top}A^2X
	\end{array}
	\right).\]
	By the Schur complement, we get that $\det(M-\lambda I) = \det (\frac{1}{n}X^{\top}D^2 X - \lambda I ) \det ( \frac{1}{n}X^{\top}A^2 X - \frac{1}{n}X^{\top}ADX\left(\frac{1}{n}X^{\top}D^2 X-\lambda I\right)^{-1}\frac{1}{n}X^{\top}DAX-\lambda I)$. Therefore the minimum eigenvalue of $M$ a positive constant if both matrices below are positive definite \[\frac{1}{n}X^{\top}D^2 X \textrm{ and }\frac{1}{n}X^{\top}A^2 X - \frac{1}{n}X^{\top}ADX\left(\frac{1}{n}X^{\top}D^2 X\right)^{-1}\frac{1}{n}X^{\top}DAX.\]
	The first matrix has clearly minimum eigenvalue a positive constant. The second matrix is equal to $\frac{1}{n}X^{\top}A(I - DX(X^{\top}D^2X)^{-1}XD)AX$ which has minimum eigenvalue positive by assumption.
\end{proof}

\begin{remark}[Smoothness of Hessian]\label{rem:smooth2}
If we want to find an upper bound on the eigenvalues of the negative Hessian by an easy argument using Lemma \ref{lem:quadratic} follows that 
\begin{align}
&\lambda_{\max} (-H _{\theta,\beta,\kappa}) \\ 
&\leq C\left\{\norm{(\beta A+D)^{-1/2}A(\beta A+D)^{-1/2}}^2_F+ d\sigma^2_{\max}((\beta A+D)^{-1/2}DX)+d\sigma^2_{\max}((\beta A+D)^{-1/2}AX)\right\}
\end{align} 
for all $(\theta,\beta,\kappa) \in \mathbb{B}$. From Lemma \ref{lem:bounds} we conclude that there exists a positive constant $C_H$ such that $\lambda_{\max} (-\frac{1}{n}H _{\theta,\beta,\kappa}) \leq C_H$.
\end{remark}

The following lemma indicates that the concentration results and the lower bound on the eigenvalues of the negative Hessian (for this section) are of the desirable order.
\begin{lemma}[Bounding the norms]\label{lem:bounds}
	The following claims hold:
	\begin{enumerate}
		\item $\norm{X}_F^2$ is $\Theta(n)$.
		\item $\frac{1}{n}\norm{X}_2^2$, $\frac{1}{n}\lambda_{\min}(X^{\top}X)$ are positive constants.
		\item $\norm{\Sigma^{1/2}A\Sigma^{1/2}}^2_F$, $\sigma^2_{\min}(\Sigma^{1/2}DX)$ and $\sigma^2_{\min}(\Sigma^{1/2}AX)$ are $\Theta(n)$.
	\end{enumerate}
\end{lemma}
\begin{proof}
	For claim 1,2, it follows from Lemma \ref{thm:feature} that $\frac{1}{n} X^{\top}X$ has minimum eigenvalue at least $\lambda_{\min}(Q)-o(1)$ and maximum eigenvalue at most $\lambda_{\max}(Q)$ with probability $1-o(1)$.
	For claim 3 we have that \[\norm{\Sigma^{1/2}A\Sigma^{1/2}}^2_F = \norm{\Sigma A}^2_F \geq \sigma^2_{\min}(\Sigma) \norm{A}_F^2 \textrm{ which is }\Theta(n)\textrm{ by assumption on } A, \Sigma.\] Moreover, $\sigma^2_{\min}(\Sigma^{1/2}DX) \geq \sigma_{\min}(\Sigma) \cdot \sigma^2_{\min}(D) \lambda_{\min}(X^{\top}X)$ which is $\Theta(n)$ with high probability (note that $\sigma_{\min}(\Sigma)$, $\sigma^2_{\min}(D)$ are positive constants). Similarly the proof goes for $\sigma^2_{\min}(\Sigma^{1/2}AX)$.
\end{proof}

We are now ready to prove the main theorem:
\begin{proof}[Proof of Theorem \ref{thm:linear}]
	For any $\delta>0$ and using Markov's inequality it follows from Lemmas \ref{lem:boundgrtheta}, \ref{lem:boundgrbeta}, \ref{lem:boundgrkappa} and Lemma \ref{lem:bounds} that \[\Pr_{\theta_0,\beta_0}\left[\norm{\nabla LL(\theta_0,\beta_0, \beta_0 \cdot \theta_0)}_2 \geq C_{\delta}\sqrt{dn}\right] \leq \delta\] for some constant $C_{\delta}$ and $\lambda_{\min}(\nabla^2 LL) \geq \frac{Cn}{d}$ for some constant $C$. We conclude from the analysis in Section \ref{sec:linear-consistency} that $\norm{(\theta_0,\beta_0) - \hat{\theta},\hat{\beta}}$ is $O\left(d\sqrt{\frac{d}{n}}\right)$ with probability at least $1-\delta$.
\end{proof}

We conclude by showing that the S-K model satisfies the assumptions we have made and hence Theorem~\ref{thm:linear} can be applied to it.
\begin{lemma}[(S–K) model satisfies the assumptions]\label{lem:SKmodel} Let $A$ be a $n \times n$ matrix such that $A_{ij} = \frac{g_{ij}}{\sqrt{n}}$ for $i<j$, where $g_{ij} \sim \mathcal{N}(0,1)$ and $A_{ji} = A_{ij}, A_{ii}=0$. Matrix $A$ satisfies the assumptions of our main Theorem \ref{thm:linear}.
\end{lemma}
\begin{proof} We assume for simplicity of the calculations that $D = c I$ for some positive constant $c$.
	
	Set $F = I - X(X^{\top}X)^{-1}X^{\top}$ ($F$ is called a hat/projection matrix, it has the property that $d$ eigenvalues are zero and the rest are one since $F^2 = F$).
	Let $B$ be a matrix with i.i.d entries $\mathcal{N}(0,1)$. It is clear that the matrix $W  = \frac{1}{\sqrt{n}}B$ satisfies the following:
	\begin{itemize}
		\item $\norm{\frac{W+W^{\top}}{\sqrt{2}}}_F^2$ is lower bounded by the sum of $\frac{n(n+1)}{2}$ i.i.d $\chi^2$ variables with mean $\frac{1}{n}$ and hence $\norm{\frac{W+W^{\top}}{\sqrt{2}}}_F^2$ is lower bounded from $\Theta(n)$ with high probability $1-o(1)$. Moreover, $\norm{\frac{W+W^{\top}}{\sqrt{2}}}_F^2 \leq 2\norm{W}_F^{2}$  and is clear that $\norm{W}_F^2$ is concentrated around $n$ ($n^2$ i.i.d variables with mean $1/n$). Thus $\norm{\frac{W+W^{\top}}{\sqrt{2}}}_F^2$ is concentrated around $\Theta(n)$. Same is true for $\norm{\frac{F(W+W^{\top})}{\sqrt{2}}}_F^2$
		\item $\norm{\frac{W+W^{\top}}{\sqrt{2}}}_2, \norm{W}_2$ are with high probability $\Theta(1)$ (it follows from semicircle law, see \cite{vershynin2010introduction}). Moreover, the same is true for $\norm{\frac{F(W+W^{\top})}{\sqrt{2}}}_2, \norm{FW}_2$.
	\end{itemize}
	Note that the reason behind the fact that multiplying by $F$ does not change the claims above is because $\sigma_j (FW) \geq \sigma_{n-d-j+1}(F) \sigma_{d+j}(W) = \sigma_{d+j}(W)$ where $\sigma_j$ denotes the $j$-th largest eigenvalue of the corresponding matrix and $n \gg d$.
	
	We first show that $FWX$ has singular values the eigenvalues $Q$ plus $o(1)$ with high probability. First let $F = R^{\top}I_dR$ where $R$ is a rotation matrix and $I_d$ is the identity matrix by setting the last d rows to all zeros. It is clear that $RW$ is also a matrix with i.i.d gaussians of mean zero and variance $1/n$ each.  Condition on $X$ (in case $X^{\top}X = Q$ then the analysis is simplified), it follows that the rows of $I_d W X$ are independent (except of the last $d$ rows that are all zeros) and each row follows a gaussian $\mathcal{N}(\vec{0}, \frac{1}{n} X^{\top}X)$. Hence using Theorem \ref{thm:feature} it follows that
	\[\norm{\frac{1}{n-d} X^{\top}W^{\top}F^2 WX - \frac{1}{n-d} X^{\top}X}_2 \textrm{ is }O\left(\sqrt{\frac{\ln (n-d)}{n-d}}\right)\]
	with probability $1-o(1)$. Finally since again by Theorem \ref{thm:feature} we get that $\norm{\frac{1}{n} X^{\top}X - Q}_2$ is $O\left(\sqrt{\frac{\ln n}{n}}\right)$ with probability $1-o(1)$, using triangle inequality we conclude that $\norm{\frac{1}{n} X^{\top}W^{\top}F^2 WX-Q}_2$ is $O\left(\sqrt{\frac{\ln n}{n}}\right)$. The claim follows by Weyl's inequality (Lemma \ref{lem:weyl}).
	
	Moreover, we prove that $\frac{1}{n}\norm{F(W+W^{\top})X}^2_2$ is $\Omega(1)$ with probability $1-o(1)$. Let us assume without loss of generality that $\frac{1}{\sqrt{n}}\norm{X}_2<\frac{1}{2\norm{FW}^2_2}$ (by appropriately rescaling $X$ with a constant).
	It holds that
	\begin{align*}
	&\lambda_{\min}\left(\frac{1}{n}X^{\top}(W^{\top}+W)F(W+W^{\top})X\right)  \\ \geq&\lambda_{\min}\left(\frac{1}{n}X^{\top}(W^{\top}FW+WFW^{\top})X\right)+ \lambda_{\min}\left(\frac{1}{n}X^{\top}(WFW + W^{\top }FW^{\top })X\right)\\\geq&
	\lambda_{\min}\left(\frac{1}{n}X^{\top}(W^{\top}FW+WFW^{\top})X\right) - \left|\lambda_{\min}\left(\frac{1}{n}X^{\top}(WFW + W^{\top}FW^{\top})X\right)\right|
	\\\overbrace{\geq}^{\textrm{Lemma }\ref{lem:usefulineq}}&
	\lambda_{\min}\left(\frac{1}{n}X^{\top}(W^{\top}FW+WFW^{\top})X\right) - 2\sigma_{\min}\left(\frac{1}{n}X^{\top}WFW X\right).
	\end{align*}
	It is clear from the analysis above that the first term is with probability $1-o(1)$ within error $O\left(\sqrt{\frac{\ln n}{n}}\right)$ from $2\lambda_{\min}(Q)$ (a).
	We analyze the other term and we get using Lemma \ref{lem:usefulineq}
	\begin{align*}
	2\sigma_{\min}\left(\frac{1}{n}X^{\top}WFWX\right) &\leq 2 \sqrt{\lambda_{\min}(\frac{1}{n}X^{\top}X)}\norm{\frac{1}{\sqrt{n}}X}_2 \norm{WFW}_2
	\\&\leq 2 (\lambda_{\min}(Q)+o(1))\norm{\frac{1}{\sqrt{n}}X}_2 \norm{FW}_2^2 \\& \leq \lambda_{\min}(Q)+o(1) \;\;\; (b).
	\end{align*}
	Finally by combining (a), (b) it holds that
	\[
	\lambda_{\min}\left(\frac{1}{n}X^{\top}(W^{\top}+W)F(W+W^{\top})X\right) \geq \lambda_{\min}(Q) - o(1),
	\]
	which is a positive constant. Hence we conclude that $\frac{1}{n}\norm{\frac{F(W+W^{\top})}{\sqrt{2}}X}^2_2$ is a positive constant.
	
	We define the matrix $A$ to be $A_{ii}=0$ (zeros in the diagonal) and $A_{ij} = \frac{W_{ij}+W_{ji}}{\sqrt{2}}$ for $i \neq j$ ($A$ is symmetric). It is clear that $A$ captures the SK model. Moreover, it is easy to show that all the diagonal entries of $\frac{W+W^{\top}}{\sqrt{2}}$ are smaller than $O(\frac{\sqrt{\log n}}{\sqrt{n}})$ with probability $1-o(1)$, hence it follows that $\norm{A - \frac{W+W^{\top}}{\sqrt{2}}}_2$ is $o(1)$ with high probability.
	
	Therefore $\sigma_{\min}(\frac{1}{n}X^{\top}A^{\top}FAX)$, $\norm{A}_2$ are positive constants and $\norm{A}_F^2$ is $\Theta(n)$, all the statements with probability $1-o(1)$ and the assumptions on matrix $A$ are satisfied for linear regression model.
	
\end{proof}

\section{Projected Gradient Descent Analysis}
\label{sec:gd-analysis}

In this Section, we will present the projected gradient descent algorithms we use for our logistic and linear regression settings.
We will use the following well known property of Projected Gradient Descent (Theorem 3.10 from \cite{bubeck2015convex}).
\begin{theorem}\label{thm:projected} Let $f$ be $\alpha$-strongly convex and $\lambda$-smooth on compact set $\mathcal{X}$. Then
projected gradient descent with stepsize $\eta = \frac{1}{\lambda}$ satisfies for $t\geq 0$
\begin{equation}
\norm{\vec{x}_{t+1} - \vec{x}^*}_2^2 \leq e^{-\frac{\alpha t}{\lambda}}\norm{\vec{x}_{1} - \vec{x}^*}_2^2.
\end{equation}
Therefore, setting $R = \norm{\vec{x}_{1} - \vec{x}^*}_2$ and by choosing $t = \frac{2\lambda \ln \frac{R}{\epsilon}}{\alpha}$ it is guaranteed that $\norm{\vec{x}_{t+1} - \vec{x}^*}_2 \leq \epsilon$.
\end{theorem}
\subsubsection{Projected Gradient Descent for Logistic Regression}
We consider the function $LPL(\theta,\beta)$ (log-pseudolikelihood as defined in Section \ref{sec:logistic}) and we would like to approximate $(\hat{\theta},\hat{\beta})$ within $\frac{1}{\sqrt{n}}$ in $\ell_2$ distance. The stepsize in Theorem \ref{thm:projected} should be $\eta=\frac{1}{\sqrt{d\Theta^2 +1}}$ by Remark \ref{rem:smooth}.

\begin{algorithm}[H]
	\label{alg:pgd-logistic}
	\KwData{Vector sample $\vec{y}$, Magnetizations $m_i(\vec{y}) = \sum_j A_{ij}y_j$, Feature vectors $\vec{x}_i$}
	\KwResult{Maximum Pseudolikelihood Estimate}
	$\beta^0 =0, \theta^0 = \vec{0}, \textrm{normgrad} = +\infty$, $\eta = \frac{1}{\sqrt{d\Theta^2 +1}}$\;
	$t = 0$\;
	\While{$\textrm{normgrad} > \frac{1}{\sqrt{n}}$}{
		$\textrm{grad}_{\theta}=0$\;
		$\textrm{grad}_{\beta} = -\frac{1}{n}\sum_{i=1}^n \left[ y_i m_{i}(\vec{y}) -  m_i(\vec{y})\tanh (\beta^t m_i(\vec{y})+\theta^{t \;\top}\vec{x}_i)\right]$\;
		\For{$k=1;k\leq d;k++$}
		{
			$\textrm{grad}_{\theta_k} = -\frac{1}{n}\sum_{i=1}^n \left[ y_ix_{i,k} - x_{i,k}\tanh (\beta^t m_i(\vec{y})+\theta^{t \;\top}\vec{x}_i)\right]$\;
			$\textrm{grad}_{\theta} = \textrm{grad}_{\theta} + \textrm{grad}^2_{\theta_k}$\;
		}
		$\textrm{normgrad} = \sqrt{\textrm{grad}^2_{\beta}+\textrm{grad}_{\theta}}$\;
		
		$\beta^{t+1} = \beta^t -  \eta\textrm{grad}_{\beta}$    \% update $\beta^t$\;
		\For{$k=1;k\leq d;k++$}
		{
			$\theta^{t+1}_k = \theta^t_k - \eta\textrm{grad}_{\theta_k}$     \% update $\theta^t_k$\;
		}
		$t = t+1$\;
		\% $\ell_2$ projection\\
		\If{$\beta^{t+1} <-B$}{
			$\beta^{t+1} = -B$\;
		}
		\If{$\beta^{t+1} > B $}{
			$\beta^{t+1} = B $\;
		}
		\For{$k=1;k \leq d;k++$}{
			\If{$\theta^{t+1}_k < - \Theta$}{
				$\theta^{t+1}_k = -\Theta$\;
			}
			\If{$\theta^{t+1}_k > \Theta$}{
				$\theta^{t+1}_k = \Theta$\;
			}
			
		}
	}
	\Return{$(\theta^t,\beta^t)$}
	\caption{Projected Gradient Descent (Logistic)}
\end{algorithm}

\subsubsection{Projected Gradient Descent for Linear Regression}
We consider the function $LL(\theta,\beta, \kappa)$ (log-pseudolikelihood as defined in Section \ref{sec:reparameter}) and we would like to approximate $(\hat{\theta},\hat{\beta},\hat{\kappa})$ within $\frac{1}{\sqrt{n}}$ in $\ell_2$ distance.
The stepsize in Theorem \ref{thm:projected} should be $\eta=1/C_H$ where $C_H$ is the constant from Remark \ref{rem:smooth2}.

\begin{algorithm}[H]
	\label{alg:pgd-linear}
 \KwData{Vector sample $\vec{y}$, Matrices $A$,$D$, Feature matrix $X$}
 \KwResult{Maximum Likelihood Estimate}
 $\beta^0 =0, \theta^0 = \vec{0}, \kappa^0 = \vec{0}, \textrm{normgrad} = +\infty$\;
 $t = 0$\;
 \While{$\textrm{normgrad} > \frac{1}{\sqrt{n}}$}
 {
 $\textrm{grad}_{\theta}=0$\;
 $\textrm{grad}_{\kappa} = 0$\;
 $\textrm{grad}_{\beta} = \frac{1}{2}\vec{y}^{\top}A\vec{y} - \frac{1}{2}tr(A(\beta^tA+D)^{-1}) - \frac{1}{2} (AX \kappa^t + DX\theta^t)^{\top}((\beta^tA+D)^{-1}) A ((\beta^tA+D)^{-1})(AX \kappa^t + DX\theta^t)   $\;
 \For{$k=1;k\leq d;k++$}
  {
 $\textrm{grad}_{\theta_k} = - \sum_{i=1}^n y_i D_{ii} x_{ik} + \sum_{i=1}^n ((\beta^t A+D)^{-1} (AX \kappa^t + DX\theta^t))_i D_{ii} x_{ik}$\;
  $\textrm{grad}_{\theta} = \textrm{grad}_{\theta} + \textrm{grad}^2_{\theta_k}$\;

  $\textrm{grad}_{\kappa_k} = - \sum_{i=1}^n \sum_{j=1}^n A_{ij}y_i x_{jk} + \sum_{i=1}^n \sum_{j=1}^n A_{ij} x_{jk} ((\beta^t A+D)^{-1} (AX \kappa^t + DX\theta^t))_i$\;
  $\textrm{grad}_{\kappa} = \textrm{grad}_{\kappa} + \textrm{grad}^2_{\kappa_k}$\;
  }
 $\textrm{normgrad} = \sqrt{\textrm{grad}^2_{\beta}+\textrm{grad}_{\theta} + \textrm{grad}_{\kappa}}$\;

$\beta^{t+1} = \beta^t -  \eta\textrm{grad}_{\beta}$    \% update $\beta^t$\;
\For{$k=1;k\leq d;k++$}
{
	$\theta^{t+1}_k = \theta^t_k - \eta\textrm{grad}_{\theta_k}$     \% update $\theta^t_k$\;
	$\kappa^{t+1}_k = \kappa^t_k - \eta\textrm{grad}_{\kappa_k}$     \% update $\kappa^t_k$\;
}
$t = t+1$\;

 \% $\ell_2$ projection\\
\If{$\beta^{t+1} <-B$}{
	$\beta^{t+1} = -B$\;
}
\If{$\beta^{t+1} > B $}{
	$\beta^{t+1} = B $\;
}
\For{$k=1;k \leq d;k++$}{
	\If{$\theta^{t+1}_k < - \Theta$}{
		$\theta^{t+1}_k = -\Theta$\;
	}
	\If{$\theta^{t+1}_k > \Theta$}{
		$\theta^{t+1}_k = \Theta$\;
	}
	\If{$\kappa^{t+1}_k < -B.\Theta$}{
		$\kappa^{t+1}_k = -B.\Theta$\;
	}
	\If{$\kappa^{t+1}_k > B.\Theta$}{
		$\kappa^{t+1}_k = B.\Theta$\;
	}
 }
}
 \Return{$(\theta^t,\beta^t)$}
 \caption{Projected Gradient Descent (Linear)}
\end{algorithm}


\newcommand{\noopsort}[1]{} \newcommand{\printfirst}[2]{#1}
\newcommand{\singleletter}[1]{#1} \newcommand{\switchargs}[2]{#2#1}




\end{document}